%% file: main.tex
\documentclass[10pt,letterpaper,twocolumn]{article}
\usepackage[margin=0.75in]{geometry}
\usepackage{times}   %
\usepackage{helvet}  %
\usepackage{courier}  %
\usepackage[hyphens]{url}  %
\usepackage{graphicx} %
\usepackage{natbib}  %
\setcitestyle{round}
\usepackage{caption} %
\usepackage{algorithm}
\usepackage{algorithmic}

\usepackage{abstract}

\input{latex_definitions/local_latex_stuff}

\usepackage{newfloat}
\usepackage{listings}
\DeclareCaptionStyle{ruled}{labelfont=normalfont,labelsep=colon,strut=off} %
\floatstyle{ruled}
\newfloat{listing}{tb}{lst}{}
\floatname{listing}{Listing}

\usepackage{authblk}

\author[1,2]{Eric Günther}
\author[3]{Balázs Szabados}
\author[1]{Kristof Meding}
\author[1,2]{Gunnar König}
\author[1,2]{Sebastian Bordt}
\author[1,2]{Ulrike von Luxburg}

\affil[1]{University of Tübingen}
\affil[2]{Tübingen AI Center}
\affil[3]{HUN-REN Institute for Computer Science and Control (SZTAKI), Budapest, Hungary
\texttt{\{eric.guenther, kristof.meding, gunnar.koenig, sebastian.bordt, ulrike.luxburg\}@uni-tuebingen.de, balazs.szabados@sztaki.hu}}

\title{We Need Explanation Cards to Connect Explanation Algorithms to the Real World}

\begin{document}

\maketitle

\begin{abstract}
Algorithmic explanations are intended to help stakeholders understand opaque algorithmic decisions, but in practice, they often fall short. First, the meaning of algorithmic explanations is often not what one might intuitively expect, so expert knowledge is required to interpret them correctly. Second, recent work has shown that popular explanation algorithms are uninformative about the behavior of complex decision functions.
Together, these issues create a gap between what explanations appear to convey and what they actually provide.
In this work, we propose {\it Explanation Cards for Explanation Algorithms}, which augment standard explanations with complementary information about robustness and validity, as well as clear instructions for interpretation. The complementary information can render otherwise uninformative explanations practically useful, while also helping to detect cases where they are not. Importantly, the interpretation instructions in explanation cards shift responsibility from users to providers: Rather than expecting users to recognize what can and cannot be concluded from an explanation, providers must make this explicit upfront. Using counterfactual explanations and SHAP as examples, we demonstrate how providers can construct explanation cards and that these cards provide users with the guidance needed for sound interpretation. We further argue that explanation cards might offer a practical way to operationalize some provisions of the EU AI Act.
Overall, explanation cards are a significant step toward making explanation algorithms fit for real-world use cases.
\end{abstract}
\newpage

\section{Introduction}

Explainable machine learning (XAI) methods are meant to help stakeholders understand opaque, black-box machine learning models. For example, XAI methods have been discussed as a tool towards enabling transparency for users affected by algorithmic decisions, and they have also been discussed as tools to debug algorithmic decisions \citep{Ribeiro16,WachterEtal17}. Meanwhile, the field of explainable machine learning faces fundamental criticism, and various authors have advocated against their use in high-stakes settings \citep{lipton2018mythos,rudin2019stop, krishnan2020against,krishna2022disagreement,freiesleben2023dear,corbett2023interrogating,schmitt2024role,deck2024critical,skirzynski2025discrimination,BorRaiLux25}.

\paragraph{Challenge of Misinterpretation.}
A prominent criticism is that algorithmic explanations lend themselves to {misinterpretation} \citep{ghassemi2021false,BorRaiLux25}, leading stakeholders to derive wrong conclusions from explanations \citep{molnar2020general,BorFinRaiLux22,freiesleben2023dear,Huang23,geirhos2024,tomaszewska2024Position,khadar2025explain}. 
One cause of misinterpretation is that the meaning of explanations is often not what one would intuitively expect \citep{Bhatt20}. Instead, detailed expert knowledge is required to translate them into justified insight \citep{khadar2025explain}. This is particularly problematic in high-stakes settings, where explanations are intended to help non-expert stakeholders understand how a machine learning model arrived at a particular decision \citep{BorFinRaiLux22}. 
For example, \citet{ghassemi2021false} notes in a medical context that the ``interpretability gap of explainability methods relies on humans to decide what a given explanation might mean''. 

\paragraph{Technical Challenges of Explaining Complex Models.}
Research on the technical properties of explanation algorithms has highlighted significant drawbacks in many popular algorithms. For example, it has been shown that the output of explanation algorithms is sensitive to seemingly irrelevant changes in the model that one aims to explain \citep{Dombrowski19,slack2020fooling}. More generally, work has shown that strong assumptions about the underlying model are required for explanations to be practically useful, and that the explanations themselves provide no means of assessing whether those assumptions are met \citep{GarLux20,BorLux_shap_2023}. In particular, \citet{gunther2025informative} has demonstrated that local feature attributions are only informative about models that are sufficiently simple. Taken together, these results suggest that {in settings where the model is complex, algorithmic explanations may not justify any intuitive interpretation at all}.

\paragraph{Explanation Cards for Explanation Algorithms.}
In this paper, we propose a pragmatic solution that addresses both the problem of model complexity and misinterpretation: {Explanation Cards for Explanation Algorithms}. Explanation Cards augment the output of standard explanation algorithms with complementary information about their robustness and validity, 
and provide clear instructions for how the output of the explanation algorithms should be interpreted. From a technical perspective, this complementary information renders the explanations informative. 
From the perspective of the stakeholder who receives the explanation card, the complementary information makes the explanation practically useful whenever possible, while also helping to detect cases where it is not. Most importantly, the instructions for interpretation shift responsibility from the user to the provider of the explanation: Rather than expecting the user to recognize what can and cannot be concluded from an explanation, providers must make this explicit upfront.\\

{\bf Our key contributions are the following: }
\begin{itemize}
    \item We suggest using explanation cards to alleviate challenges in interpreting the output of common explanation algorithms. To construct such cards, one needs to first derive technical guarantees about the specific explanation, and then translate these technical guarantees to natural-language statements that outline to a user which concrete interpretations can and cannot be drawn from the given explanation. 
    \item We derive the necessary mathematical guarantees to construct two example explanation cards: an explanation card for counterfactual explanations for loan decisions  (Section~\ref{sec:counterfactual-explanations}), and an explanation card for SHAP explanations in medical diagnosis (Section~\ref{sec:SHAP}). 
    \item In a legal analysis, we show how the use of explanation cards could help to satisfy the transparency and explainability obligations set by the EU AI Act (Section~\ref{sec:legal}). 
\end{itemize}

\section{Related Work}
\label{sec:related-work}

In this paper, we focus on post-hoc explanation algorithms \citep{Ribeiro16, Lundberg17, WachterEtal17}. They are widely used, for example, for model debugging \citep{adebayo2020debugging}, but also in high-stakes scenarios like medicine \citep{borys2023explainable}, recourse recommendations \citep{karimi2022survey} or detecting discrimination \citep{agarwal2018automated}. For a general overview of the big field of explainable AI (XAI), we refer the reader to \citet{molnar2025}.

\paragraph{Counterfactual Explanations.} The first explanation algorithm discussed in this work is {counterfactual explanations} \citep{WachterEtal17}. Counterfactual explanations have gained significant attention because they align with how humans naturally reason about cause and effect \citep{miller2019explanation,byrne2019counterfactuals}. They have also been proposed as a recourse mechanism, for example, in credit lending \citep{ustun2019actionable}. However, it is also known that counterfactuals lack robustness: An adversary can train models with unstable counterfactuals \citep{slack2021counterfactual} and generally, small changes in the observation can lead to different explanations \citep{Fokkema23}. Furthermore, changes in one feature may have unexpected downstream causal effects on other features \citep{karimi2021algorithmic, karimi2022survey,konig2023improvement,konig2025performative} and a counterfactual may be invalidated if changed slightly \citep{dominguez2022adversarial}. More recently, it was shown that counterfactuals are not informative for arbitrarily complex functions \citep{gunther2025informative}. These insights have led to a number of attempts to make counterfactuals more robust \citep{artelt2021evaluating, zhang2023density, leofante2024promoting}. Particularly related to our work are regional counterfactuals \citep{fernandez2020random,vannostrand2023facet}, which have been proven to be easier to understand for a subject \citep{vannostrand2024actionable}.  %

\paragraph{SHAP Explanations.} The second explanation algorithm discussed in this work is {SHAP} \citep{Lundberg17}. SHAP has become one of the most widely used post-hoc explanation algorithms due to its ease of use and model-agnostic applicability \citep{lundberg2018consistent, molnar2025}. However, the algorithm has also been frequently criticized \citep{Kumar20}: Similar to counterfactual explanations, SHAP is sensitive to off-manifold behavior \citep{frye2020shapley, taufiq2023manifold} and can be manipulated by an adversary \citep{slack2020fooling, hwang2025shap}. Furthermore, the presence of variable interactions limits the interpretation of the SHAP attributions, which has given rise to algorithmic variants that also attribute interactions \citep{lundberg2018consistent, zhang2021interpreting, BorLux_shap_2023, fumagalli2023shap}. Techniques like these also exist for global methods that aim to explain the data \citep{konig2024disentangling}. But even for models without interactions, SHAP values remain uninformative for complex functions due to their potential instability \citep{gunther2025informative}. We hence present several SHAP explanations, presenting their stability, and give additional information that reveals whether they connect to the prediction, which is only possible if the model does not exhibit strong interactions. %

\paragraph{Legal Discussion Around Explainability.} In the context of high‑stakes decisions, the connection to the {legal domain} has also been widely examined: Explainability requirements have been discussed in relation to the General Data Protection Regulation (GDPR) \cite{WachterEtal17,hamon2022bridging,selbst2018meaningful}, various drafts and the final version of the AI Act \cite{panigutti2023role,kaminski2025right}, and their combined implications \cite{hacker2020varieties,ebers2020regulating,nannini2024habemus}.

\paragraph{Explanation Cards} are inspired by model cards \citep{mitchell2019model} and datasheets \citep{gebru2021datasheets}. They are also closely related to saliency cards, proposed by \citet{boggust2023saliency}. However, we develop explanation cards for tabular data rather than for images, and tailor them to a non-expert audience. Explanation cards move beyond previous approaches that enhance explanation algorithms with robustness and stability measures \citep{bhatt2021evaluating, velmurugan2021evaluating, Dasgupta22, dhurandhar2024trustRegions} because they both enrich the explanations with additional information {\it and} focus on how this information can be communicated to stakeholders in real-world applications. In this respect, explanation cards provide a step towards user-centric explainability \citep{miller2019explanation, liao2021human, freiesleben2023dear, schor2024meaningful, hullman2025explanationsmeansend}.

\section{Problem Setup}
\label{sec:problem-setup}

We consider machine learning algorithms that make predictions in social contexts. Specifically, we consider the setup where an algorithm is trained on {\bf tabular data} to predict individual characteristics, such as a medical diagnosis or the probability of loan repayment. Each dimension of the input encodes a different property of a person, for example, their age, income, or blood pressure. The machine learning algorithm is used to learn a {\bf decision function} $f: \mathbb{R}^d \to \mathbb{R}$, outputting a score or a probability $f(x_0)$ for an individual $x_0$. This score can be transformed to a binary prediction (“receives the loan” or “does
not receive the loan”). In addition to the prediction, the provider of the machine learning system provides an explanation. In this work, we consider {\bf counterfactual explanations}, introduced by \citet{WachterEtal17}, and {\bf SHAP}, introduced by \citet{Lundberg17}. Counterfactual explanations provide another datapoint close to $x_0$ that has a different prediction. SHAP provides a scalar feature importance value for each dimension of the input.

Our main research question is whether individuals who are affected by machine predictions can make use of the explanations to answer %
questions that are \textbf{practically relevant} \citep{liao2021human,BorRaiLux25}. For example, a loan applicant might want to know whether earning more would be sufficient to be granted the loan, and a doctor might want to know how a change in a patient's blood pressure would change the predicted outcome. 

We also ask whether the provided explanations satisfy the {\bf informativeness} condition from \citet{gunther2025informative}. Informativeness is a mathematical notion that quantifies whether an explanation provides any information about the behavior of the decision function. Informativeness can be seen as a necessary condition on explanation algorithms, as a non-informative explanation does not provide a robust description of the predictive behavior of the decision function. A key insight of this framework is that explanations tend to be informative if one can construct local regions in which the explanations behave benignly, for example, they only vary by a small amount. These regions will show up on our explanation cards below, for example, as "region of stability" in the case of counterfactuals.

In the mathematical parts of the paper, we consider a dataspace $\Xcal\subseteq\R^d$ with a probability distribution $\mathrm{P}$. For the set of all features, we write $[d]:=\{1,...,d\}$. We denote the features of a point $x$ by $(x^{(1)},...,x^{(d)})$. If we consider a subset of features $S=\{j_1,...,j_{|S|}\}\subseteq[d]$ of some point $x\in\Xcal$, we write $x^{(S)}:=(x^{(j_1)},...,x^{(j_{|S|})})$. We aim to explain the prediction of the function $f$ at the point $x_0\in\Xcal$. We assume that $f$ is a real-valued scoring function from which we can derive a final binary prediction by thresholding, $t$ denotes  the threshold with respect to which the prediction is derived.

\section{Explanation Cards for Counterfactual Explanations}
\label{sec:counterfactual-explanations}

\begin{figure*}[t]
     \centering
     \begin{subfigure}[b]{0.32\textwidth}
         \centering
         \includegraphics[width=\textwidth]{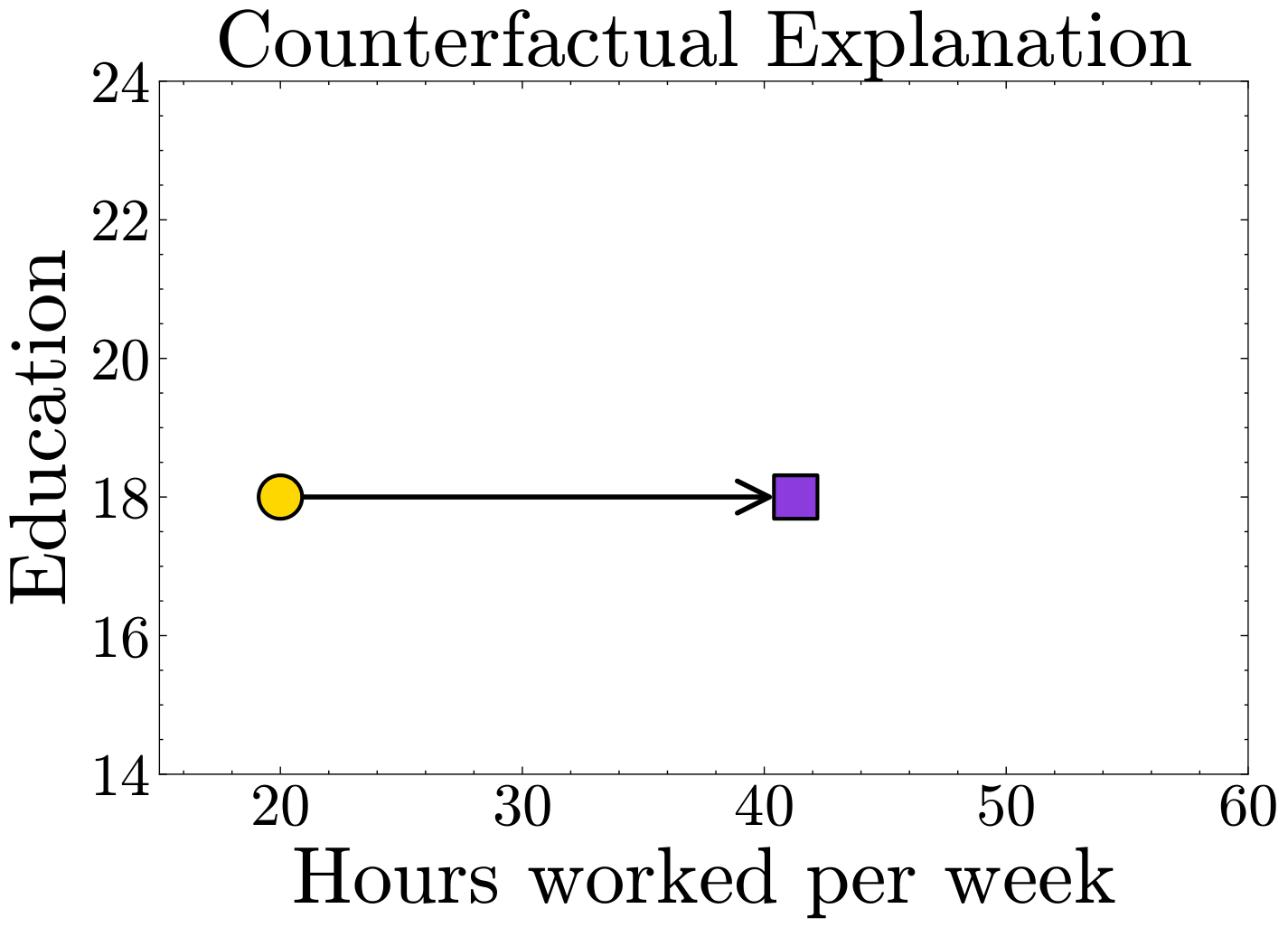}
         \caption{A recipient obtains a counterfactual explanation.}
         \label{fig:counterfactual_plain}
     \end{subfigure}
     \hfill
     \begin{subfigure}[b]{0.32\textwidth}
         \centering
         \includegraphics[width=\textwidth]{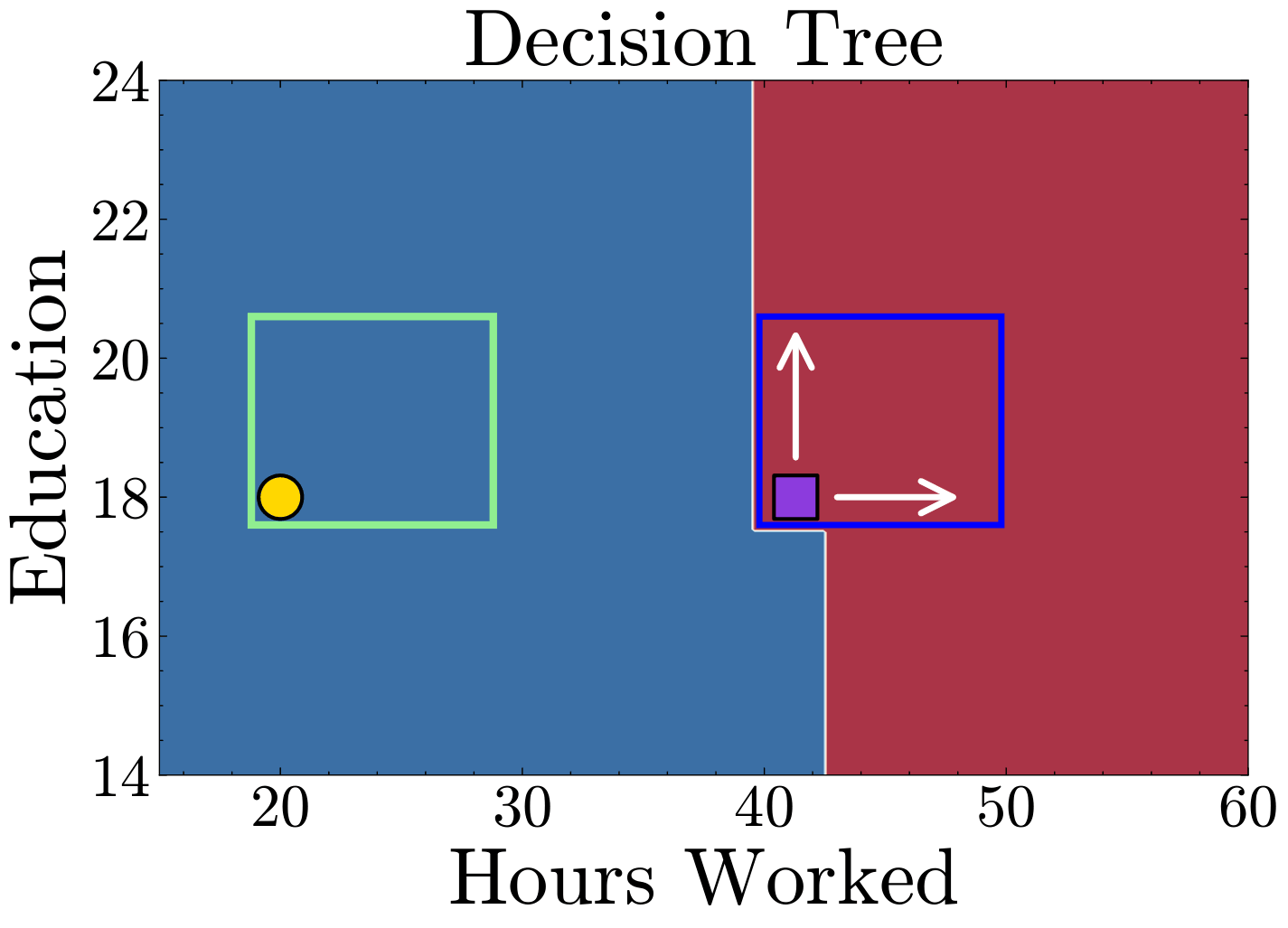}
         \caption{What the recipient believes the decision boundary looks like.}
         \label{fig:decision_tree}
     \end{subfigure}
     \hfill
     \begin{subfigure}[b]{0.32\textwidth}
         \centering
         \includegraphics[width=\textwidth]{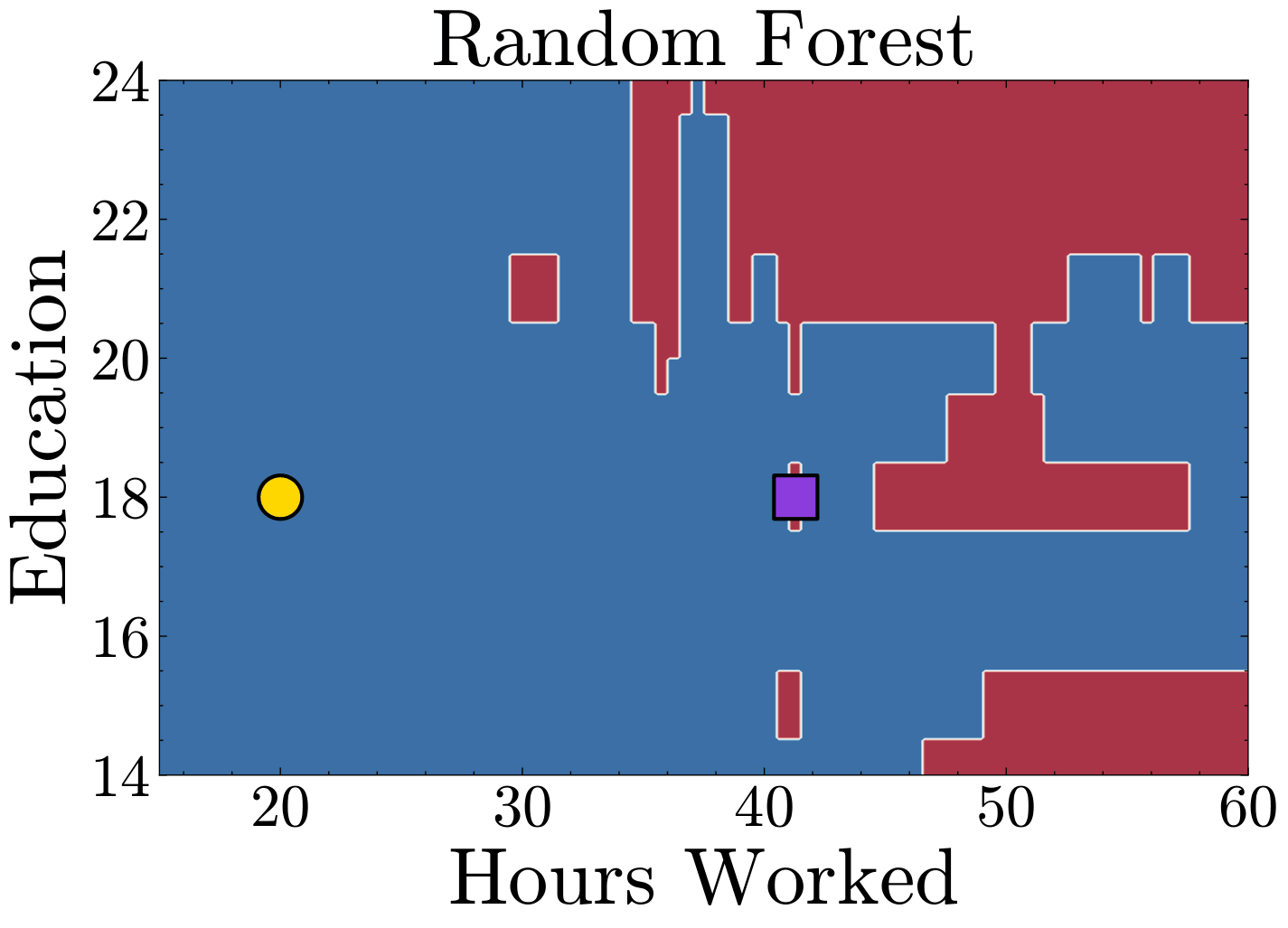}
         \caption{How the decision boundary may actually look like.}
         \label{fig:random_forest}
     \end{subfigure}\\
     \includegraphics[width=0.9\textwidth]{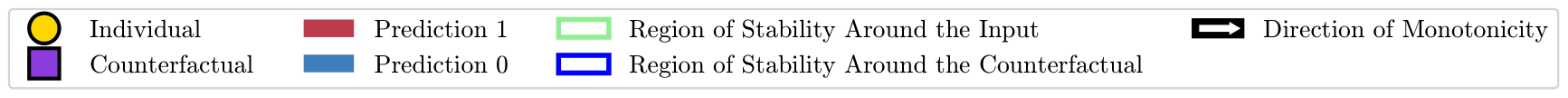}
        \caption{{\bf Intuitive Interpretations of Counterfactual Explanations May Be Invalid for Complex Models.} The left panel illustrates a {\color{yellow!80!black} data point} and its \textcolor{violet}{counterfactual explanation}. It summarizes the factual knowledge that a user has after receiving their explanation. The middle plot illustrates a decision boundary as it might implicitly be assumed by the user who tries to interpret the received explanation. It satisfies the properties of monotonicity (white arrows), stability around the counterfactual (\textcolor{blue}{blue box}), and stability around the input (\textcolor{green}{green box}). The right panel shows the decision boundary of a complex model as it might be used by the bank. Here, the counterfactual does not satisfy stability or monotonicity. For this model, the user's intuitive interpretations are invalid. 
        }
        \label{fig:random_forest_vs_decision_tree}
\end{figure*}

\begin{figure*}[t]
\begin{tcolorbox}[title={Explanation Card for Counterfactual Explanations in Loan Approval, Targeting the Loan Applicant}, enhanced,
    colback=white, colframe=black!60,
    boxrule=0.5pt, sharp corners,
    width=\textwidth]
\hlA{An automated system has denied your loan application because you do not have savings. To receive a loan upon reapplication, you are suggested to open a savings account and increase your savings to 100 DM.}\\[0.5em]
\hlB{The action suggested by the counterfactual explanation remains valid as long as you do not open a checking account, and your savings are}\textit{\hlB{ above }}\hlB{100 DM. Your housing and possible other credits do not influence the prediction.}\\[0.5em]
\hlD{However, changes in other features, such as your job, the number of people financially dependent on you, etc. 
may alter the loan decision. In these cases, the reapplication may no longer be successful. }\\[0.5em]
\hlC{Similar applicants who are younger than 26.5 years, do not have a checking account, and have savings of less than 100 DM are denied the loan for the same reason and are advised of the same changes.}

\end{tcolorbox}
\caption{\legendbox{color1} Counterfactual. \quad
\legendbox{color2} Stability of the Counterfactual. \quad
\legendbox{color3} Stability around the input. \quad
\legendbox{color4} Further caveat. \quad\\[8pt]
{\bf Explanation Card for a Counterfactual Explanation in Loan Approval.} The explanation card provides the machine decision, the counterfactual, as well as the necessary information about the stability that is required for interpretation. It also provides information about the {\it monotonicity} of the counterfactual. In this example, the classifier is a decision tree trained on the South German Credit dataset \citep{south_german_credit_522}.}
\label{fig:explanation_card_counterfactual_german_credit}
\end{figure*}

In this section, we introduce Explanation Cards for counterfactual explanations. We begin by discussing the intuitive interpretations that humans may assign to counterfactual explanations (Section \ref{sec:intuitive_cf}). We then show that these interpretations are invalid for decision functions that have a complex decision boundary (Section \ref{sec:invalid_cf}). Explanation Cards can overcome this problem: They provide the user with an interpretation that is valid for the given decision function (Section \ref{sec:explanation_cards_cf}).

\subsection{Interpreting Counterfactual Explanations}
\label{sec:intuitive_cf}

We begin by discussing the real-world interpretations that humans may assign to counterfactual explanations. For illustration, we consider a simplified setup where a machine learning model determines loan approval based on three attributes: age, income, and current account balance \citep{WachterEtal17,mothilal2020dice,karimi2021algorithmic,verma2024counterfactual}. As an affected person, the loan applicant has under certain circumstances the right to obtain clear and meaningful explanations according to Article 86 of the European AI Act. Concretely, assume that a 32-year-old applicant with a yearly income of \$50,000 and a current account balance of \$4,000 is denied a loan. To offer recourse, the bank provides a counterfactual explanation,a nearby datapoint from a different class. 

\begin{quote}
{\it An automated system has denied your loan application because your current account balance is too low. To receive the loan upon reapplication, you are suggested to increase your account balance to \$14,000.}
\end{quote}

The applicant now attempts to implement the counterfactual explanation. To do so, she will {\it interpret} the counterfactual explanation, meaning that she attempts to infer real-world behaviors that would lead to loan approval \citep{WachterEtal17, ustun2019actionable,karimi2021algorithmic,BorRaiLux25}. For example, the applicant may conclude that increasing her yearly income to \$55,000 will allow her to save \$5,000 per year, so that at age 34, she will have saved the required \$10,000. In doing so, the applicant makes implicit assumptions about the explanation and the machine learning algorithm's behavior. Concretely, the applicant assumes that:

\begin{enumerate}
    \item \textbf{Monotonicity:} Although the explanation requires an increase of \$10,000 in balance, a growth of more than \$10,000 would also result in a positive outcome. In general, the counterfactual would still be valid if the yearly income and account balance are increased further.
    
    \item \textbf{Stability of the Counterfactual:} The provided explanation will still be valid if, over the course of implementing the explanation in the real world, the values of some of the other features change. In our example, the applicant assumes that the counterfactual explanation remains valid even after she has become two years older.
\end{enumerate}

In addition, we argue that the applicant might also assume that
\begin{enumerate}\setcounter{enumi}{2}
    \item \textbf{Stability Around the Input:} Other applicants with a similar background are suggested to implement similar changes. For example, people in the 30-35 age group, with a salary of \$45,000-55,000, and a balance lower than \$5,000, would be granted the loan after increasing their savings by approximately \$10,000.
\end{enumerate}

The fact that the applicant draws conclusions about the real world from the explanations is not specific to counterfactuals: Work in human-centered explainability has shown that such interpretations are what human decision makers are ultimately interested in \cite{liao2021human,raees2024explainable}.

\subsection{Intuitive Interpretations of Counterfactual Explanations May Be Invalid for Complex Models}
\label{sec:invalid_cf}

Whether the intuitive interpretations from the previous Section~\ref{sec:intuitive_cf} are valid depends strongly on the geometry of the decision boundary -- that is, on the specific machine learning model that the bank employs \citep{laugel2019issues,artelt2021evaluating,zhang2023density}. This is illustrated in Figure \ref{fig:random_forest_vs_decision_tree}. In the left panel of the Figure \ref{fig:random_forest_vs_decision_tree}, we illustrate an individual and a corresponding counterfactual explanation. In the middle panel of Figure \ref{fig:random_forest_vs_decision_tree}, we illustrate the decision boundary as it might be implicitly imagined by a user who interprets the explanation as discussed in the previous Section \ref{sec:intuitive_cf}. For this decision boundary, the user's intuitive interpretations would be valid. In the right panel of Figure \ref{fig:random_forest_vs_decision_tree}, we depict how the decision boundary of the machine learning model may actually look like. Concretely, we depict the decision boundary of a random forest classifier, a popular machine learning algorithm for tabular datasets. For this decision boundary, the user's intuitive interpretations are invalid: The counterfactual is not monotone, because moving further into the hours worked direction invalidates the counterfactual. The counterfactual is also not stable around the input, because slightly increasing the hours worked moves the counterfactual in the direction of the negative class. Similarly, the counterfactual itself is not stable. 

It is important to stress that the user who receives the explanation cannot know whether the bank's function behaves as in the middle plot (which would be good) or as in the right plot (in which case the explanation, albeit formally being a correct counterfactual, might be misleading to the user). Of course, one could argue that a bank might be unlikely to use a funny decision function as in the right plot. But the key point is that the user cannot take this for granted, which directly leads to our explanation cards: Whether the bank uses a nice model, as in the middle of the plot, or a funny model, as in the right plot, will become apparent from the explanation card that we introduce below.

Whether a counterfactual explanation is stable also has strong mathematical implications. We can formalize the above conditions as mathematical assumptions on the explanations. In each of them, $t$ is the decision threshold, meaning all points $x\in\Xcal$ yielding $f(x)<t$ are assigned the same label as $x_0$, and all points $x\in\Xcal$ yielding $f(x)\ge t$ are assigned the other label.

\begin{enumerate}
    \item \textbf{Monotonicity:} Let $x_C$ be a counterfactual explanation, and let $\mathbb{X}\subseteq\mathcal{X}$ be such that $x_C\in\mathbb{X}$ and $\forall x\in\mathbb{X}:f(x)\geq t$. The set $\mathbb{X}$ is called monotonically increasing in feature $j$ if for all $x\in\mathbb{X}$ it holds that $f(x+\lambda e_j)\geq t$ for any $\lambda\geq0$, where $e_j$ is the $j$-th standard basis vector.
    \item \textbf{Stability of the Counterfactual:} A counterfactual explanation $x_C$ is stable if there exist vectors $l,u\in\bar{\mathbb{R}}^d$ such that $l<x_C<u$ and $f(x)\geq t$ for all $l<x<u$, where the inequalities are coordinate-wise and $\bar{\mathbb{R}}:=\mathbb{R}\cup\{-\infty,\infty\}$.
    \item \textbf{Stability Around the Input:} A counterfactual explanation $x_C$ for the data point $x_0$ is stable around the input if there exist vectors $l,u\in\bar{\mathbb{R}}^d$ such that $l<x_0<u$ and $f(x)<t\leq f(x+(x_C-x_0))$ for all $l<x<u$. Again, the inequalities are coordinate-wise and $\bar{\mathbb{R}}:=\mathbb{R}\cup\{-\infty,\infty\}$.
\end{enumerate}

Now, we can show that counterfactuals that do not provide a region of stability do not provide reliable information about the behavior of the decision function:

\begin{theorem}[Non-Intuitive Counterfactual Explanations Are Not Informative, Informal]\label{thm:non-intuitive_implies_non-informative}
    Let $x_0\in\mathcal{X}$ be such that $f(x_0)<t$, and let $x_C\in\mathcal{X}$ be a counterfactual explanation for $x_0$. If $x_C$ is neither a stable counterfactual nor stable around the input, then it is non-informative. The monotonicity of $x_C$ alone, that is, $\mathbb{X}=\{x_C\}$ is not sufficient for informativeness.
\end{theorem}

For a discussion of the mathematical concept of ``informative'' explanations, we refer to \citet{gunther2025informative}; a formal statement of the theorem and the proof can be found in Appendix~\ref{proof:non-intuitive_implies_non-informative}.

\subsection{Explanation Cards Provide the User with Valid Interpretations}
\label{sec:explanation_cards_cf}

In this section, we propose an explanation card for counterfactual explanations. The explanation card provides information about the counterfactual's stability, enabling the user to interpret it. In a nutshell, {\bf the provider of the explanation, who has access to the decision boundary, enriches the explanation with additional information and clearly communicates to the user which interpretations are valid.}

Figure \ref{fig:explanation_card_counterfactual_german_credit} depicts our suggestion for an explanation card for a counterfactual explanation. In the first paragraph, the explanation card provides the machine decision and the counterfactual explanation. In the second paragraph, the explanation card provides information about the stability of the counterfactual: In the provided example, it communicates to the user that the counterfactual remains valid as long as they sufficiently increase their savings and do not open a checking account. The third paragraph provides the applicant with additional information about how changes in other features, for example stemming from downstream causal effects of the implemented action, may affect the credit decision. In the last paragraph, the explanation card provides additional information about the stability around the input: Individuals in the same age group who have a similar financial background obtain the same prediction and suggested change as the applicant. Additional examples of explanation cards for counterfactuals are provided in Appendix \ref{app:counterfactual-expl-cards}.

At an intuitive level, the benefits of the additional information in the explanation card are obvious: The applicant now knows what
conclusions can and cannot be drawn from the explanations. As we show
in the following theorem, these benefits can also be described mathematically:

\begin{theorem}[Stable Counterfactual Explanations Are Informative, Informal]\label{thm:intuitive_implies_informative}
    Let $x_0\in\mathcal{X}$ be such that $f(x_0)<t$, and let $x_C\in\mathcal{X}$ be a counterfactual explanation for $x_0$. If $x_C$ is a stable counterfactual or stable around the input, then it is informative.
\end{theorem}

The formal theorem and its proof can be found in Appendix \ref{app:coounterfactuals:intuitive_implies_informative}. %
An important property of the explanation card is that it provides the size of the regions within which the counterfactual remains stable. Mathematically, it is possible that the counterfactuals are stable only to tiny changes in the input variables; this would probably be considered useless by most people. Explanation cards reveal this and make the actual amount of stability transparent to the user.

\section{Explanation Cards for SHAP}
\label{sec:SHAP}

As follows, we propose Explanation Cards for SHAP, one of the most popular explanation algorithms \citep{Lundberg17}. 
SHAP decomposes the prediction at point $x_0$
into a sum of individual contributions of each feature, plus the
expected prediction as a baseline. Mathematically, this means
\[f(x_0) = \Phi_1(x_0) + ... + \Phi_d(x_0)+\E(f(X)),\]
where $\Phi_j(x_0)$ is the SHAP value of feature $j$ at $x_0$. For a detailed mathematical definition, see Appendix \ref{app:SHAP-def}. We consider the common interventional version of SHAP, which is also implemented in the Python package \texttt{shap} \citep{Lundberg17, janzing2020feature}.

In Section \ref{subsec:intuitive-interpretations}, we consider a hypothetical example where a doctor uses SHAP to understand a model and derive intuitive interpretations they might draw from the SHAP values. As we will see in Section
\ref{subsec:underlying-assumptions}, these interpretations rely on
assumptions about the underlying model, which may not hold in complex
decision models.  When presented with the corresponding explanation
card, introduced in Section \ref{subsec:shap-explanation-card}, the doctor is able to assess whether their intuitive interpretations are
justified.

\subsection{Intuitive Interpretations of SHAP Values}
\label{subsec:intuitive-interpretations}
 
Consider the following scenario: A doctor is using an AI model to predict whether a patient has chronic heart disease. As such, the doctor could be a deployer in the sense of Article 3(4) of the AI Act and the provider of the system must ensure transparency according to Article 13. Assume the doctor is given SHAP explanations to gain insight into the model's prediction behavior.
\footnote{Note that in this setting, the doctor does \emph{not} aim to learn about real-world relationships between the features and the target. Rather, the doctor wants to decide whether to trust the model's decision. Assuming that interventional SHAP provides insight into real-world statistical or causal relationships is another common misinterpretation; see \citet{chen2020true, janzing2020feature,freiesleben2024scientific} for a discussion.}
For example, suppose the patient in question is $43$ years old, and the corresponding SHAP value for age is $-0.22$. From this single SHAP value, there may arise the first misinterpretation:
\begin{enumerate}
    \item \label{shap-interpretation-0}\textbf{Stability:} Similar patients of a similar age have a similar SHAP value.
\end{enumerate}
This interpretation seems natural. However, it does not necessarily hold. A common way to get a more representative picture, is by plotting several SHAP values in a SHAP dependence plot. It shows SHAP values across many instances by age, allowing for a better understanding of their stability.
\begin{figure}[t]
    \centering
    \includegraphics[width=0.9\linewidth]{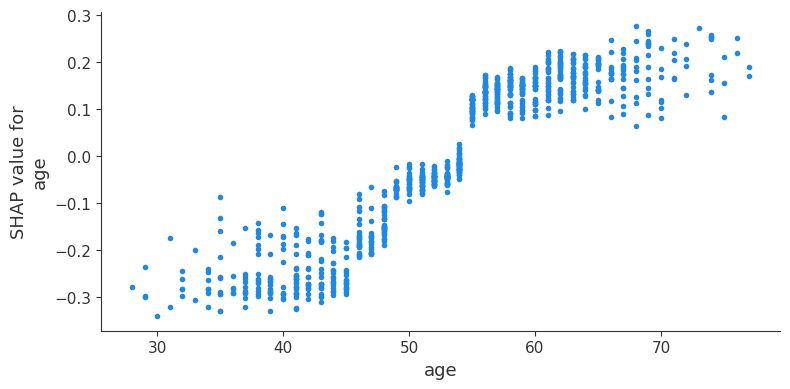}
    \caption{A classical SHAP dependence plot for age, generated by the shap package.}
    \label{fig:shap-dp}
\end{figure}
Figure \ref{fig:shap-dp} shows such a dependence plot across the dataset, confirming that most patients around age 43 have similar negative SHAP values, ranging from roughly $-0.3$ to $-0.1$.

However, even this SHAP dependence plot might still lead to misinterpretation. In particular, the doctor may conclude:

\begin{enumerate}
    \setcounter{enumi}{1}
    \item \label{shap-interpretation-1}\textbf{Extrapolation:} The SHAP value of the 43-year-old patient is negative and the one of a 53-year-old patient is zero. Hence, the patient would have a higher predicted risk if they were older.
    \item \label{shap-interpretation-2}\textbf{Monotonicity:} The SHAP values in the plot increase with age. Hence, the predicted risk increases with age.
\end{enumerate}
In general, neither of these interpretations holds. In the following section, we explain why and subsequently address the problem using our explanation card in Section \ref{subsec:shap-explanation-card}.

\begin{figure*}
    \centering
    \begin{subfigure}[t]{0.22\linewidth}
        \includegraphics[width=\linewidth]{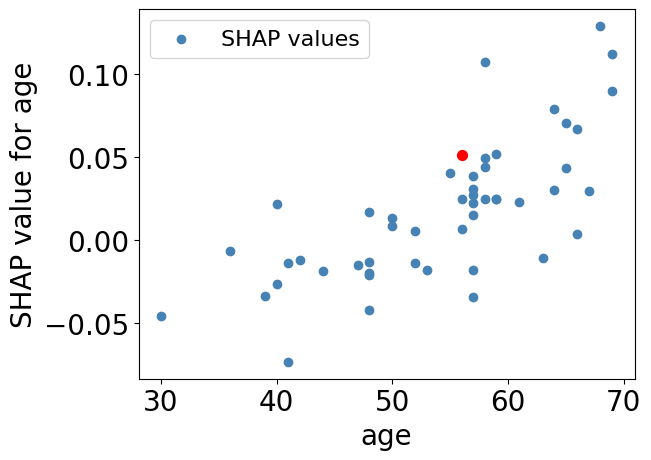}
        \caption{Unstable SHAP values.}
        \label{fig:SHAP-DP chaotic}
    \end{subfigure}
    \hfill
    \begin{subfigure}[t]{0.22\linewidth}
        \includegraphics[width=\linewidth]{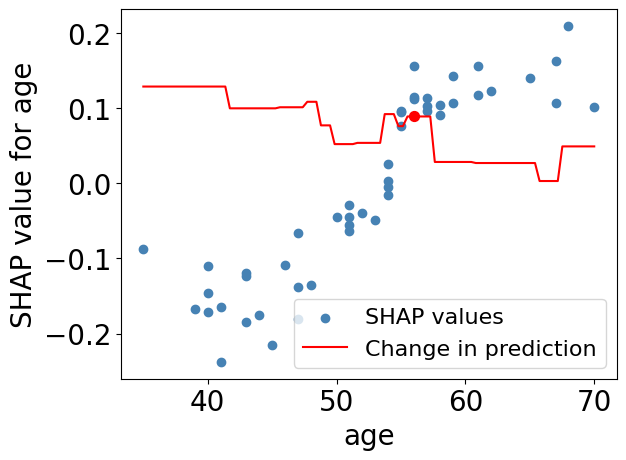}
        \caption{Stable but misleading SHAP values.}
        \label{fig:SHAP-DP bad ICE curve}
    \end{subfigure}
    \hfill
    \begin{subfigure}[t]{0.22\linewidth}
        \includegraphics[width=\linewidth]{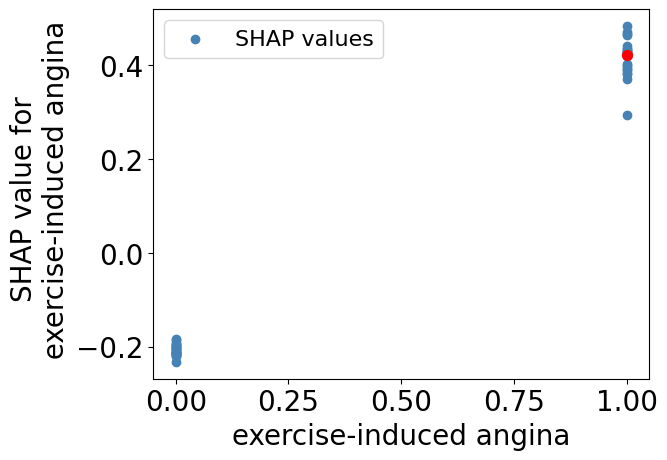}
        \caption{Ambiguous SHAP values.}
        \label{fig:SHAP-DP-angina}
    \end{subfigure}
    \hfill
    \begin{subfigure}[t]{0.22\linewidth}
        \includegraphics[width=\linewidth]{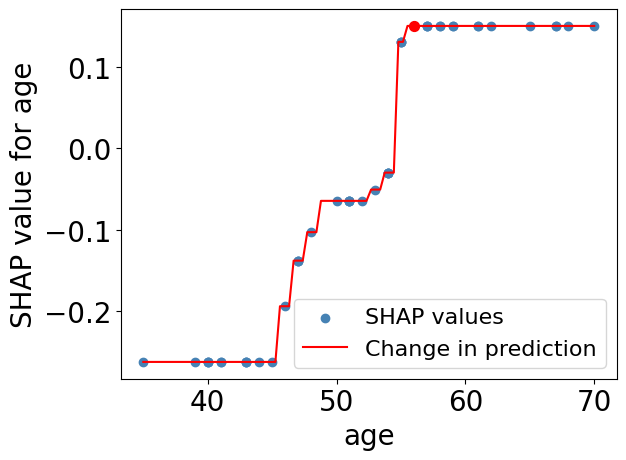}
        \caption{Additive model.}
        \label{fig:SHAP-DP GAM}
    \end{subfigure}
    \caption{
    SHAP dependence plots for one datapoint (red) and its 50 nearest neighbors in the input space (blue).
    Figures~\ref{fig:SHAP-DP bad ICE curve} and~\ref{fig:SHAP-DP GAM}
    additionally show the prediction change when varying age while keeping all
    other features fixed.
    Figure~\ref{fig:SHAP-DP chaotic} illustrates unstable SHAP values.
    Figure~\ref{fig:SHAP-DP bad ICE curve} shows a seemingly stable dependence
    plot that nevertheless misrepresents the prediction change.
    This discrepancy is caused by other features, like exercise-induced angina, whose SHAP values may also depend on age, as shown in Figure~\ref{fig:SHAP-DP-angina}.
    For additive models, shown in
    Figure~\ref{fig:SHAP-DP GAM}, SHAP values are functions of the corresponding
    feature and exactly recover the prediction change.
    }
    \label{fig:shap-illustration}
\end{figure*}

\subsection{The Intuitive Interpretations May Be Invalid for Complex Models}
\label{subsec:underlying-assumptions}

\begin{figure*}[t]
\tcbset{
  cardstyle/.style={
    enhanced,
    colback=white,
    colframe=black!60,
    boxrule=0.5pt,
    sharp corners,
  }
}
\begin{tcolorbox}[
    cardstyle,
    title={Explanation Card for SHAP Explanations in Medical Diagnosis, Targeting the Doctor},
    width=\textwidth
]

    \hlA{The patient is classified as sick with a predicted score of $0.565$. In general, patients with positive scores are classified as sick; the higher the score, the more certain the classification algorithm.} \\[0.5em]
    \hlB{The plots below show the SHAP values of patients in a local region around the explained instance: \\% and the range of the relevant feature is colored in the plot.
    }
    \begin{tabularx}{\linewidth}{rXrXrX}
        \hlB{Age:} &\hlB{29-62} & \hlB{Sex:} & \hlB{Male} & \hlB{Resting BP:} & \hlB{112-160 mmHG}\\
        \hlB{Cholesterol:} & \hlB{157-325 mg/dl} & \hlB{Fasting blood sugar:} & \hlB{$< 120$ mg/dl} & \hlB{Max heart rate:}  & \hlB{120-202}\\
        \hlB{ST depression:} & \hlB{0-3} & \hlB{ST slope:} & \hlB{upsloping or flat} 
    \end{tabularx}
\tcblower
\begin{tcolorbox}[
    cardstyle,
    title={Feature: Age},
    sidebyside,
    sidebyside gap=3mm,
    righthand width=0.55\textwidth
]
\includegraphics[width=\linewidth]{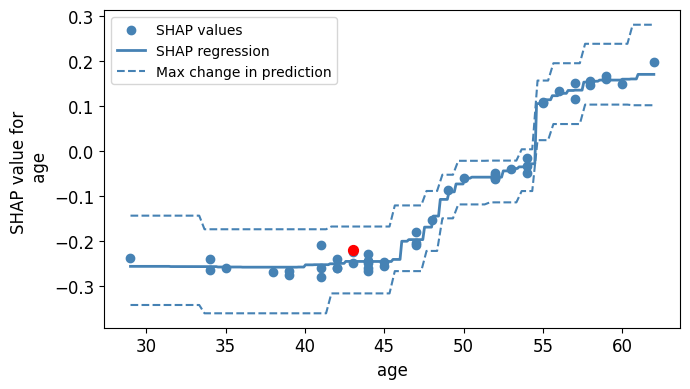}
\tcblower
\hlA{SHAP attributes a value of $-0.22$ to the patient's age ($43$), see red dot in the plot.}\\[0.5em]
\hlC{The SHAP values or other patients in the local region (blue dots) and their regression (blue line) roughly reflect how age affects the prediction. The dotted lines show the maximal/minimal change in prediction if age is varied.} %
\end{tcolorbox}
\end{tcolorbox}
\caption{\legendbox{color1} Machine decision and SHAP value. \;
\legendbox{color2} Local region. \;
\legendbox{color3} Feature effect on prediction. \\ \\
{\bf Explanation Card for a SHAP Explanation for a doctor debugging a prediction:} 
In the explanation card, the SHAP explanations are augmented with 
human-intelligible interpretations, thereby helping prevent misinterpretation. They are split into general information, the prediction and a local region, and additional feature-wise plots.
Depending on the model and observation, the card may also reveal that SHAP dependence plot does not reflect the prediction behavior as the dotted envelope is too large.}
\label{fig:explanation_card_shap_good}
\end{figure*}

The presented interpretations do not hold for arbitrarily complex models: Interpretation \ref{shap-interpretation-0} needs the model to be stable, and Interpretation \ref{shap-interpretation-1} and \ref{shap-interpretation-2} need the model to be additive. In the following, we formalize the interpretations, discuss them further, and connect them to mathematical theory.

\paragraph{SHAP Values Are Stable Only if the Model is Stable}
We begin by formalizing the concept of stability for SHAP values. 

\begin{definition}[Local Stability]
    For an instance $x_0\in\R^d$, we call its SHAP values \textbf{locally stable} with respect to a local region $U$ around $x_0$ and tolerance thresholds $\eps_1,...,\eps_d$ if $|\Phi_j(x)-\Phi_j(x_0)|\le \eps_j$ for $j=1,...,d$ and all $x\in U$.
\end{definition}

It is clear that the SHAP values can only be stable if the model itself is stable: As it holds $f(x)=\Phi_1(x)+...+\Phi_d(x)+\E(f(X))$, the model can vary by at most $\eps_1+...+\eps_d$ within a region $U$ if our SHAP values are stable within $U$ with respect to tolerance thresholds $\eps_1,...,\eps_d$. As one is usually not given any stability guarantees for a model, the stability of the SHAP values is not guaranteed either. As a counterexample, consider Figure \ref{fig:SHAP-DP chaotic}. It shows the SHAP values for age of a patient (red dot) as well as their 50 nearest neighbors (blue dots, representing a local region $U$). As we see, there is no stability in the SHAP value, not even for those neighbors of the same or very similar age, who have SHAP values with positive as well as negative signs.
Hence, it is always important to consider several SHAP values, for example, in a SHAP dependence plot, if one wants to understand the prediction mechanism of the model. A single SHAP value is not representative of that. 

Local stability is not just an intuitive definition; it is also grounded in theory. As shown in \citet{gunther2025informative}, SHAP values are not informative unless we have further assumptions about our model class. However, when SHAP values come with stability guarantees, they do become informative with respect to the definition in \citet{gunther2025informative} as the following theorem states.  

\begin{theorem}[Locally Stable SHAP Values Are Informative, Informal]\label{Thm-shap-informative}
  Consider the explanation \linebreak
  $\Ecal(f, x_0)= ((\Phi_1(x_0),...,\Phi_d(x_0)), U,
  (\eps_1,...,\eps_d))$, which, given a point $x_0$, returns the SHAP
  values of $x_0$, a neighborhood $U$ around $x_0$ and tolerance
  thresholds $\eps_1,...,\eps_d$ such that the SHAP values are locally
  stable with respect to $U$ and $\eps_1,...,\eps_d$. Then,
  $\Ecal$ is an informative explanation if $U$ has positive
  probability and the tolerance thresholds are sufficiently small.
\end{theorem}
A formal version of the theorem, including exact bounds on $\eps_j$ as well as the proof, can be found in Appendix \ref{app:proof-shap-informative}. But even if SHAP values are locally stable, a SHAP dependence plot can still be misleading if our model is not additive. %

\paragraph{Interpretation \ref{shap-interpretation-1} and \ref{shap-interpretation-2} Only Hold if the Model Is Additive.}
Interpretations \ref{shap-interpretation-1} and \ref{shap-interpretation-2} seem natural: If the plot shows an increase from age 43 to age 53, this might suggest that the risk of disease would be higher if our patient were older. Generally, the increase in the plot might raise the impression that the risk increases with age. However, Figure \ref{fig:SHAP-DP bad ICE curve} shows a counterexample illustrating that neither of the two has to hold. There, we see the SHAP value of a patient in red as well as their nearest neighbors. Also here, it seems that increasing the patient's age increases the prediction -- but the opposite is true: The red curve shows the change in prediction if we vary the age of the patient represented by the red dot with all other features fixed, revealing that the prediction actually decreases with age. The reason for this is that the output is decomposed into the SHAP values by $f(x)=\Phi_1(x)+...+\Phi_d(x)+\E(f(X))$ and the SHAP dependence plot only shows $\Phi_1(x)$. If the SHAP values of all other features stayed the same when age is varied, this would imply that the SHAP plot exactly reflects the prediction behavior of the function (up to a constant). However, each SHAP values generally depends on all features, rather than just on its own feature, as has been pointed out
repeatedly in the literature \citep{lundberg2018consistent, zhang2021interpreting, BorLux_shap_2023, fumagalli2023shap}. Hence, $\Phi_2,...,\Phi_d$ may also change when age is varied, making the prediction change in ways not reflected in the plot. For example, consider the feature of whether a patient has exercise-induced angina. One could imagine that the presence of such angina is only considered risky by the model if the patient is of a certain age. Figure \ref{fig:SHAP-DP-angina} shows the SHAP values of the same patients as in Figure \ref{fig:SHAP-DP bad ICE curve}, but this time for the feature ``exercise induced angina''. And indeed, we observe that the SHAP values of other patients with exercise-induced angina vary from 0.3 to 0.5. Thus, this SHAP value is influenced by other features, possibly also by age. Similarly, a change in age could also cause a strong change in \textbf{all} other SHAP values, explaining why the red curve in Figure \ref{fig:SHAP-DP bad ICE curve} shows such a different prediction behavior than what the SHAP plot suggests.

However, there is an assumption on the underlying model under which each SHAP value is invariant to change in other features: We call a function \textbf{additive} if it can be written as a
sum of functions that only depend on one feature. That is,
$f(x)=f_1(x^{(1)}) +...+ f_d(x^{(d)})$ for component functions $f_1,...,f_d:\R\to\R$. To get rid of ambiguous constants in the components, one often separates the expected value and mean-centers the components, leading to $f(x)=f_1(x^{(1)}) +...+ f_d(x^{(d)})+\E(f(X))$. \citet{BorLux_shap_2023} shows that in this case, every SHAP value is given by $\Phi_j(x)=f_j(x^{(j)})$. In particular, it only depends on the corresponding feature and is hence invariant to changes in other features. Reversely, if each SHAP value only depends on its own feature, the function $f$ is additive as it is given by $f(x) = \Phi_1(x^{(1)})+...+\Phi_d(x^{(d)}) + \E(f(X))$.
For an illustration, see Figure \ref{fig:SHAP-DP GAM}. It shows a SHAP dependence plot for an additive model, and again, a red curve illustrating the change in prediction as the age of the person represented by the red dot is varied. Unlike before, we now see that all SHAP values lie on a line (and are hence invariant to changes in other features) and precisely reflect the prediction behavior. This means that in the case of additive models, SHAP values do exactly what one would like them to do: They split the function value into a sum of feature contributions, which are completely determined by the feature itself. In particular, the SHAP dependence plot also describes the change in prediction with respect to that feature, rather than just the SHAP value. 

In practice, it is unrealistic to expect our function to be additive. %
Nevertheless, for the doctor, it might be sufficient to understand the predicting behavior only in a local region, meaning for a given patient and similar patients. Also, we can further relax the problem by considering \emph{approximate} statements. If in a local region, all SHAP values are mostly invariant to change in other features, Interpretation \ref{shap-interpretation-1} and \ref{shap-interpretation-2} will be approximately valid in that region. This leads to the following definition.

\begin{definition}[SHAP Value Invariant to Change in Other Features]\label{def:weakly-interacting}
  We say that the $j$-th SHAP value of a function $f$ is approximately invariant to change in other features in a local region $U\subseteq \Xcal$ and with respect to some error $\eps>0$ if the expected unexplained variance of the SHAP value given its feature is less than $\eps$: 
	$$\E_{X\sim U}\left(\Var_{X\sim U}\left(\Phi_j(X)\mid X^{(j)}\right)\right)\le\eps.$$
\end{definition}

The unexplained variance given feature $X^{(j)}$ measures how much variance there is left when we control feature $j$, and thus, how much the SHAP value varies due to a change in other features. With the aid of those quantities, one can now mathematically quantify to what degree a given SHAP value is invariant to change in the other features.\\
As it turns out, this implies approximate additivity of the model in that local region, analogously to the global and exact case. %

\begin{theorem}[Local Approximate Invariance of SHAP Values Implies Local Approximate Additivity]\label{Thm-GAM-expl-var}\
	If every SHAP value of a function $f$ is approximately invariant to change in other features with respect to some local region $U$ and some error $\eps>0$, $f$ can be approximated by an additive model $g(x)=g_1(x^{(1)})+...+g_d(x^{(d)})+\E(f(X))$ on $U$ with mean squared error at most 
	$$\E_{X\sim U}\left((f(X)-g(X))^2\right) \le d^2\eps.$$ 
    The components of $g$ are given by $g_j(x^{(j)})=\E(\Phi_j(X)\mid X^{(j)}=x^{(j)})$.
\end{theorem}
The proof is in Appendix \ref{app:proof-GAM-expl-var}.\\
Although most models might not be globally additive, they might still yield an approximately additive structure in local regions, which would make an approximate, local version of Interpretation \ref{shap-interpretation-1} and \ref{shap-interpretation-2} realistic.
The additive model $g$, as defined in Theorem \ref{Thm-GAM-expl-var}, can be estimated by regressing the SHAP values with a one-dimensional model. The resulting mean-squared error can serve as a measure for how close the model is to being additive in the region of interest, and hence, how strongly a SHAP dependence plot links to the prediction there. This allows for choosing a region accordingly before further studying the SHAP values with our explanation card.

\subsection{Explanation Cards Provide the User with Valid Interpretations}
\label{subsec:shap-explanation-card}

In this section, we propose an explanation card for SHAP, targeted towards a doctor aiming to understand a prediction.
The explanation card is presented in Figure \ref{fig:explanation_card_shap_good}. It reports the SHAP values along with the information that the user needs to arrive at correct interpretations.
In the header, the explanation card reports a local region around the explained instance in which the subsequent analyses are performed. Here, we consider the hyperbox defined by the 50 nearest neighbors with respect to the maximum-norm.
Below, the card shows the SHAP values of patients within this region in a dependence plot, separately for each feature. Due to space constraints, we only present this for the feature ``age''. For a full card, see Figure \ref{fig:explanation_card_shap_good_full} in Appendix \ref{app:shap-expl-card}. In addition to the SHAP dependence plot, we display a regression of the SHAP values by the corresponding feature as an estimate for the effect of the that feature on the prediction. Also, for every point in the region, we vary the age and investigate the resulting change in prediction. We then take the maximum and minimum of the change in prediction over all points in the region, leading to the dotted lines in the plot. That means that if a feature value is replaced, the change in prediction stays within the two lines. They thus give an error bound for 
Interpretation \ref{shap-interpretation-1} and \ref{shap-interpretation-2}. If the envelope between the two lines is too large (see Figure \ref{fig:SHAP-DP bad with envelope}), the plot cannot be related to the prediction behavior and interpretations like \ref{shap-interpretation-1} and \ref{shap-interpretation-2} become invalid.
For more technical details on how the card is derived from the SHAP explanations, see Appendix \ref{app:implementation-shap}.

In the example presented in Figure \ref{fig:explanation_card_shap_good}, the explanation card 
confirms an approximately additive structure with a small envelope, admitting approximate versions of Interpretation \ref{shap-interpretation-1} and \ref{shap-interpretation-2}. In cases where they do not hold, like in Figure \ref{fig:SHAP-DP bad ICE curve}, the explanation card reveals this: Figure \ref{fig:SHAP-DP bad with envelope} shows the same points as in Figure \ref{fig:SHAP-DP bad ICE curve} but with the dotted lines, which form a big envelope.
\begin{figure}[t]
    \centering
    \includegraphics[width=0.7\linewidth]{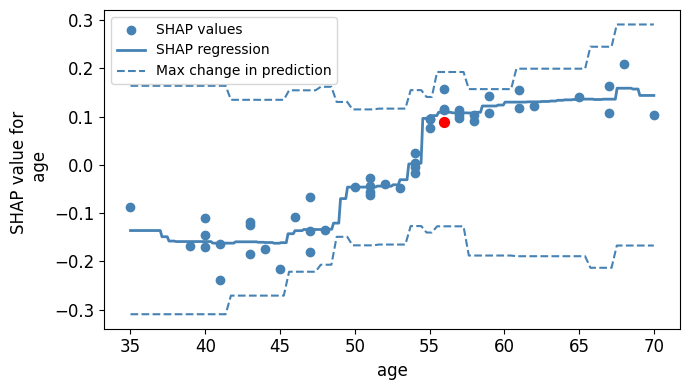}
    \caption{The fact that the dotted envelope is so large reveals that the SHAP dependence plot (blue points) does NOT reflect how the prediction function changes with age.}
    \label{fig:SHAP-DP bad with envelope}
\end{figure}
Another negative example of a SHAP explanation card with a huge dotted envelope is shown in Figure \ref{fig:explanation_card_shap_bad} in Appendix \ref{app:shap-expl-card}. It is important to highlight that an explanation card is not guaranteed to make sense of every explanation. 
Instead, it can tell us, for a given explanation, \emph{whether} one can derive a specific intuitive interpretation from it, and if so, it can translate the mathematical information to human-understandable language. Especially in the case of SHAP dependence plots, complex models often carry interactions, making it impossible to connect the plot to the prediction. The main point of this section and our card is to make these cases transparent, rather than to suggest a recipe for making SHAP values universally meaningful.

\section{Explanation Cards and Transparency in the AI Act}
\label{sec:legal}

\paragraph{General Setup of the EU AI~Act.} Within the European Union, the AI Act (AIA, Regulation (EU) 2024/1689, as amended by Regulation (EU) 2026/1744)  sets specific rules for the use of AI and distinguishes providers, who develop an AI system, or have one developed, and place it on the market or put it into service under their own name or trademark (Art. 3(3) AIA), from deployers, who use a system under their authority in a professional capacity (Art. 3(4) AIA). Both definitions are simplified; see \citet{kaminski2025american} for an introduction to the AIA. Particular attention is given to systems that are classified as high-risk (Art. 6 AIA). These must meet the requirements of Arts. 8 to 15 AIA, which aim to protect health, safety and fundamental rights (Recital 66 AIA; in detail \citet{paul2024european}). Those requirements apply from 2 December 2027 to stand-alone Annex III systems and from 2 August 2028 to systems
embedded in Annex I products (Art. 113 AIA as amended).

\paragraph{Transparency Requirements of Providers. }
Art. 13 AIA governs transparency from provider to deployer, with the aim of mitigating
the risks of high-risk systems (Recitals 66 and 72 AIA) and providing usable explanations
\cite{nannini2023explainability}. High-risk systems must be ``sufficiently transparent to
enable deployers to interpret a system's output and use it appropriately'' (Art. 13(1) AIA);
providers are bound to this via Art. 16, point (a) AIA. The instructions for use must
therefore state, where applicable, the technical capabilities and characteristics relevant
to explaining the output (Art. 13(3)(b)(iv) AIA) and the information deployers need to
interpret that output and use it appropriately (Art. 13(3)(b)(vii) AIA). Throughout, the
focus is the deployer's understanding of the system's capabilities and the limitations of
its performance (Art. 13(3)(b), Recital 72 AIA). Instructions for use are a familiar
transparency instrument in EU product safety law under the New Legislative Framework
\cite{ortigosa2025transparency}.

\paragraph{Explanation Cards could help address AI Act requirements.}
In practice, providers will rely heavily on harmonised standards. Currently being developed by CEN-CENELEC, these standards will guide practical implementation (see amendment \cite{M613}, \citet{kilian2025european}); once their references are published in the Official Journal of the European Union, following these standards will establish a presumption of conformity (Art.~40(1) AIA). From our perspective, Explanation Cards could help address the requirements outlined in Article 13 AIA. As presented in Section \ref{sec:SHAP}, Explanation Cards offer multiple benefits in the context of Art. 13 AIA. They provide a clear, accessible, and comprehensive way to communicate explanations about the system’s outputs, focusing on the individual decisions made by the system. As a result, Explanation Cards help deployers understand the capabilities of the system. Additionally, Explanation Cards also highlight the limitations of the explanations themselves, offering a more nuanced perspective. This, in turn, leads to a better understanding when assessing the system’s performance and its limitations, as deployers gain greater insight into the entirety of the system. 

\paragraph{Transparency Requirements Towards Individuals Affected by AI.}
Beyond the transparency requirements in Art. 13 AIA, Explanation Cards could be helpful for other legal transparency requirements to affected persons (See Section \ref{sec:counterfactual-explanations}). While under Art. 13 AIA it is the explanation capability of the card that is relevant, the card could also serve as an instance-level artifact for an affected person. For example, they may help deployers fulfill the duty under Art.~86 AIA, which entitles a person affected by a decision taken by the deployer on the basis of the output of an Annex~III high-risk system (excluding point~2 thereof), and producing legal effects or a similarly significant impact that the person considers adverse to their health, safety or fundamental rights, to clear and meaningful explanations of the role of the system in the decision-making procedure and of the decision's main elements. Since that duty falls on the deployer and applies only where the right does not already follow from other Union law (Art.~86(3) AIA), Explanation Cards provided under Art.~13 AIA could serve as an input from which the deployer builds the explanation to the end user. Likewise, solely automated decision-making under Art.~22 in conjunction with Art.~15(1)(h) GDPR requires some form of explanation to the data subject (CJEU Case C-203/22 \textit{Dun \& Bradstreet}, ECLI:EU:C:2025:117). On the relationship of these different ``affected-person'' explainability requirements, see \citet{nannini2024habemus,kaminski2025right}.

Presenting an output together with its explanation card in a concise and understandable format can support clear communication of both the explanation and its limitations to professionals and layperson.

\section{Discussion}
\label{sec:discussion}

\paragraph{Explanation Cards to Communicate Valid Interpretations.}
{In this work, we propose Explanation Cards for Explanation Algorithms,
an approach that makes explicit which interpretations and conclusions can be drawn from particular post-hoc explanations --- and which ones cannot.} The goal of the card is to connect the mathematically abstract concept of local post-hoc explanations to real-world applications and to avoid misinterpretations of explanations.  
In order to derive reliable interpretations, explanation providers need to run additional analyses (for
example, evaluate the function's or explanation's behavior in a neighborhood of the
point of interest). In making these explicit,  explanation cards provide additional
information that goes beyond the standard output of particular explanation
algorithms. Their generation can be automated by using text templates and simply filling in the computed values of the concrete application.

\paragraph{Making the Limitations of Interpretations Explicit.} {If an explanation
does not admit the interpretations that humans tend to make, it is
important to point this out to prevent false conclusions.} 
 For instance, the region of stability around a counterfactual explanation might be tiny, in which case it would be hard for an affected individual to change their features such that they are within this region (as in the right panel in Figure 
\ref{fig:random_forest_vs_decision_tree}). 
 Or the dotted envelope around the SHAP dependence plot might be so large that the plot cannot properly connect to the prediction (see Figure \ref{fig:SHAP-DP bad with envelope}).

\paragraph{Concrete Examples: Counterfactuals and SHAP.}
Unlike related
approaches, for example, datasheets for datasets
\citep{gebru2021datasheets}, there is no single explanation card that
works for all explanation algorithms. The card needs to be designed for a particular explanation algorithm and a particular use case, taking into account the goals, desired interpretations, and background knowledge of the explanation receiver. 
We present two exemplary explanation cards: counterfactual explanations for loan decisions and SHAP explanations for medical diagnosis.  
Compared to SHAP, Counterfactuals have a very intuitive definition and meaning, causing the explanation cards to be easy to communicate to lay persons. For SHAP, on the other hand, it is less clear what the cards should contain, and the content naturally becomes more technical. We are curious whether simpler cards will be developed in future research.

\paragraph{Related Approaches. }
The regions of stability around the input point from Section \ref{sec:invalid_cf} are similar in spirit to anchors \citep{Ribeiro18}. We explicitly acknowledge this connection: In our view, counterfactual explanations should provide anchors that describe their region of validity.
In the case of SHAP, the dotted lines are constructed using shifted ICE curves. A similar line of work suggests finding
regions of the input space that are free of interactions
\citep{gkolemis2024effector,herbinger2024decomposing}. These methods
may potentially be used in implementations of the SHAP explanation
card to find a suitable local region. 

\paragraph{Legal Considerations.} Explanations are especially important when it comes to high-stakes applications, particularly in societal context. The use of AI Systems for some of these decisions is regulated through the EU AI Act, which lays out requirements regarding transparency and explainability. We believe that Explanation Cards might provide a path forward to satisfy these requirements. 

\paragraph{Future Work.} In future work, intuitive interpretations that users assign to explanations could be studied systematically via user studies, together with the ability of Explanation Cards to reduce misinterpretations. Also, Explanation Cards for further explanation algorithms and use cases are left to be developed.

\section*{Acknowledgements}
This work has been supported by the German Research Foundation through the Cluster of Excellence “Machine Learning - New Perspectives for
Science" (EXC 2064/1 number 390727645) and Project 560788681, the Carl Zeiss Foundation through the CZS Center for AI and Law, and the International Max Planck Research School for Intelligent Systems
(IMPRS-IS).
Balázs Szabados was supported by the ``Robust Uncertainty Quantification for Learning and Control'' ADVANCED research project of the National Research, Development and Innovation Office of Hungary (NKFIH), grant number 153390.

\bibliographystyle{plainnat}
\bibliography{references}
\clearpage
\onecolumn
\appendix

\section{Proofs for Counterfactual Explanations}

\subsection{Proof of Theorem~\ref{thm:non-intuitive_implies_non-informative}}\label{proof:non-intuitive_implies_non-informative}

\begin{theorem}[Non-Intuitive Counterfactual Explanations Are Not Informative, Formal Version]
    Let $\mathcal{X}\subseteq\mathbb{R}^d$ be the data space with probability distribution $P$ that has a positive density. Let $\mathcal{F}$ be the class of all continuous functions with values in $[0,1]$. If counterfactual explanations are neither stable nor stable around the input, then they are not informative for any $f\in\mathcal{F}$, any $x_0\in\mathcal{X}$, and any $n\in\mathbb{N}$. In particular, the monotonicity of $x_C$ alone, that is, $\mathbb{X}=\{x_C\}$ is not sufficient for informativeness.
\end{theorem}

\begin{proof}
    Let $x_0\in\mathcal{X}$ be the point to be explained, and let $x_C\in\mathcal{X}$ be its counterfactual explanation. First, we prove that, without stability, counterfactual explanations are generally not informative. Let $\Fcal$ be the class of continuous functions mapping from $\mathbb{R}^d$ to $[0,1]$, and let $n\in\mathbb{N}$ be arbitrary. It is to be shown that $R_n(\Fcal_{\mathrm{explain}}^{x_0})=R_n(\Fcal_{\mathrm{predict}}^{x_0})$.

    Let us fix a sample $x_1,\dots,x_n$ with labels $\sigma_1,\dots,\sigma_n$. We will show that for every $g\in\Fcal_{\mathrm{predict}}^{x_0}$ there exists $\tilde g\in \Fcal_{\mathrm{explain}}^{x_0}$ so that $$\sum_{i=1}^n\sigma_ig(x_i)=\sum_{i=1}^n\sigma_i\tilde g(x_i).$$
    By showing this, the equality also holds when taking the supremum over the two function classes. Finally, taking the expectations over $x$ and $\sigma$ proves the result.

    It is easy to see that the above equality is always achievable. Let $\tilde g$ be such that $\tilde g(x_i)=g(x_i)$ for every $i\in\{0,\dots,n\}$, and $\tilde g(x_C)=\tau\geq t$. Since continuous functions can interpolate any finite number of points, such a function $\tilde g$ exists and, by construction, is a member of $\Fcal_{\mathrm{explain}}^{x_0}.$

    Now we show that the monotonicity of $x_C$ alone does not guarantee informativeness. If the counterfactual explanation is monotonous in feature $j$, it means that $f(x_C+\lambda e_j)\geq t$ for any $\lambda\geq0$. We will again show that for every $g\in\Fcal_{\mathrm{predict}}^{x_0}$ there exists $\tilde g\in \Fcal_{\mathrm{explain}}^{x_0}$ so that $$\sum_{i=1}^n\sigma_ig(x_i)=\sum_{i=1}^n\sigma_i\tilde g(x_i).$$
    To this end, we show the existence of a function $\tilde g$ for which $\tilde g(x_i)=g(x_i)$ for every $i\in\{0,\dots,n\}$ and $\tilde g(x_C+\lambda e_j)=\tau\geq t$ for any $\lambda\geq0$. Because the data distribution is continuous, it has probability $0$ that any Rademacher variable $x_i$ satisfies $x_i=x_C+\lambda e_j$ for some $\lambda\geq0$. Since $\{x_0,\dots,x_n\}\cup\{x\in\mathcal{X}\mid\exists \lambda\geq0:x=x_C+\lambda e_j\}$ is a closed set, due to the Tietze extension theorem such a function $\tilde g$ exists with probability $1$, and is a member of $\Fcal_{\mathrm{explain}}^{x_0}.$
\end{proof}

\subsection{Proof of Theorem~\ref{thm:intuitive_implies_informative}}
\label{app:coounterfactuals:intuitive_implies_informative}

\begin{theorem}[Stable Counterfactual Explanations Are Informative, Formal Version]\label{proof:intuitive_implies_informative}
    Let $\mathcal{X}\subseteq\mathbb{R}^d$ be the data space with probability distribution $P$ that has a positive density. Let $\mathcal{F}$ be the class of all continuous functions with values in $[0,1]$. If a counterfactual explanation $x_C\in\mathcal{X}$ for the data point $x_0\in\mathcal{X}$ is stable or stable around the input, then it is also informative for any $f\in\mathcal{F}$, and any $n\in\mathbb{N}$.
\end{theorem}

\begin{proof}
    We only prove that the stability of the counterfactual implies informativeness. Stability around the input, implying informativeness, can be proven similarly.

    Let $x_0\in\mathcal{X}$ be the point to be explained, and let $x_C\in\mathcal{X}$ be its counterfactual explanation. If $x_C$ is stable, it means that there exist vectors $l,u\in\bar{\mathbb{R}}^d$ such that $l<x_C<u$ and $f(x)\geq t$ for all $l<x<u$. Let $\mathcal{F}$ be the class of continuous functions mapping from $\mathbb{R}^d$ to $[0,1]$, and let $n\in\mathbb{N}$ be arbitrary. In order to show that $R_n(\Fcal_{\mathrm{explain}}^{x_0})<R_n(\Fcal_{\mathrm{predict}}^{x_0})$, it is enough to prove that $R_n(\Fcal_{\mathrm{explain}}^{x_0}\mid A)<R_n(\Fcal_{\mathrm{predict}}^{x_0}\mid A)$ for some event $A$ with positive probability by Lemma 21 of \citet{gunther2025informative}.

    Let $A$ be the event that the Rademacher variables satisfy $\forall i\in[n]: l<x_i<u$ and have labels $\forall i\in[n]: \sigma_i=-1$. Because the data distribution is absolutely continuous, and the set $\{x\in\mathcal{X}\mid l<x<u\}$ has positive volume, event $A$ has non-zero probability. Since it holds that $\forall g\in\Fcal_{\mathrm{explain}}^{x_0}: g(x_i)\geq t$, and functions in $\Fcal_{\mathrm{predict}}^{x_0}$ may have values arbitrarily close to $0$ at the Rademacher points $x_i$, it is easy to see that we have $R_n(\Fcal_{\mathrm{explain}}^{x_0}\mid A)<R_n(\Fcal_{\mathrm{predict}}^{x_0}\mid A)$, which was to be shown.
\end{proof}

\section{Explanation Cards for Counterfactual Explanations}
\label{app:counterfactual-expl-cards}

Below, we show a another example of an explanation card for counterfactual explanations.

\FloatBarrier

\begin{figure*}[h]
\begin{tcolorbox}[title={Explanation Card for Counterfactual Explanations in Income Prediction, Targeting the Individual}, enhanced,
    colback=white, colframe=black!60,
    boxrule=0.5pt, sharp corners,
    width=\textwidth]
\hlA{An automated system has predicted low income for you because your educational attainment and weekly working hours are too low. To receive a favorable prediction, you are suggested to obtain a master's degree beyond your bachelor's degree, and work 32 hours weekly instead of 25.}\\[0.5em]
\hlB{The action suggested by the counterfactual explanation remains valid as long as your education does not exceed a master's degree, and you work }\textit{\hlB{at least}}\hlB{ 32 hours per week. Your marital status does not influence the prediction.}\\[0.5em]
\hlD{However, changes in other features, such as occupation, class of work, etc., may alter the prediction. In these cases, the prediction may no longer be favorable.}\\[0.5em]
\hlC{Similar applicants in the 33-52 age group with a bachelor's degree, who work 25-31 hours per week, are predicted to have low income for the same reason and are advised of the same changes.}
\end{tcolorbox}
\caption{\legendbox{color1} Counterfactual. \quad
\legendbox{color2} Stability of the Counterfactual. \quad
\legendbox{color3} Stability around the input. \quad
\legendbox{color4} Further caveat. \quad\\[8pt]
{\bf Explanation Card for a Counterfactual Explanation in Income Prediction.} The explanation card provides the machine decision, the counterfactual, as well as the necessary information about the stability that is required for interpretation. It also provides information about the {\it monotonicity} of the counterfactual. In this example, the classifier is a decision tree of depth 10 trained on ACSIncome.}\label{fig:explanation_card_counterfactual_acsincome2}
\end{figure*}
\FloatBarrier

\section{Formal Definition and Proofs for SHAP}

\subsection{Definition of SHAP}
\label{app:SHAP-def}
Consider a function $f$ and a probability distribution $\mathrm{P}$ on $\Xcal\subseteq\R^d$. For attributing every single feature, one first defines a value function $v:\mathcal{P}([d]) \times \Xcal \to \R$, where $\mathcal{P}$ denotes the power set. Given a subset of features $S\subseteq[d]$ and a point $x\in\Xcal$, this value function measures the average value of the function $f$ in $x$ if only the features $S$ are known. The most common value function, which is also considered in this paper, is the interventional one, see \citet{Lundberg17}. It is given by
$$v(S,x):=\E_{X^{(\bar{S})}}\left(f\left(x^{(S)},X^{(\bar{S})}\right)\right),$$
meaning it averages out the remaining features using the marginal expectation. The SHAP values are then computed by the formula
$$\Phi_j(x) = \sum_{S\subseteq [d]\setminus j} {\frac{|S|!(d-|S|-1)!}{d!}}\cdot \Big(v(S\cup j, x)-v(S,x)\Big).$$ 
It can be shown that the SHAP values based on the interventional value function indeed satisfy 
\[f(x) = \Phi_1(x) + ... + \Phi_d(x)+\E(f(X)),\]
see \citet{Lundberg17}.

\subsection{Formal Statement and Proof of Theorem \ref{Thm-shap-informative}}
\label{app:proof-shap-informative}

\begin{theorem}[Formal Version of Theorem \ref{Thm-shap-informative}]
    Consider a dataspace $\Xcal\subseteq\R^d$ with a probability distribution $\mathrm{P}$ that has a positive density. Let $\Ecal(f,x_0)=((\Phi_1(x_0),...,\Phi_d(x_0)), U, (\eps_1,...,\eps_d))$ be the explanation which returns a function's SHAP values, a region $U$ around $x_0$ and tolerance thresholds $\eps_1,...,\eps_d$ such that the SHAP values of $x_0$ are locally stable with respect to $U$ and $\eps_1,...,\eps_d$. Let further $\Fcal$ be the class of all trees or all continuous functions with values in $[0,1]$. If $\eps_1+...+\eps_d<f(x_0)$ or $\eps_1+...+\eps_d<1-f(x_0)$, the explanation  $\Ecal(f,x_0)$ is informative for all $x_0\in\Xcal$, all $f\in\Fcal$ and all $n\in\mathbb{N}$.
\end{theorem}

\begin{proof}
    Let $n\in\mathbb{N}$ be arbitrary. We know that $\Phi_1(x)+...+\Phi_d(x) = f(x)-\E(f(X))$. Like that, the SHAP values give us the expected value of $f$. Furthermore, calling $\eps:=\eps_1+...+\eps_d$, we can conclude that $f(x)\in [f(x_0)-\eps,f(x_0)+\eps]$ for all $x\in U$. %
    So for all $x\in U$, we particularly know
    \begin{align*}                          \sup_{g\in\Fcal_{\mathrm{explain}}^{x_0}} \frac{1}{n}\sum_{i=1}^n g(x_i)\le f(x_0)+\eps
    \end{align*}
    and
    \begin{align*}
        \sup_{g\in\Fcal_{\mathrm{explain}}^{x_0}} \frac{1}{n}\sum_{i=1}^n -g(x_i)\le -(f(x_0)-\eps).
    \end{align*}
    
    Now consider the event $A$ that $x_1,...,x_n$ all fall into $U$ and that $\sigma_1=...=\sigma_n=1$ (if $\eps<1-f(x_0)$) resp. $\sigma_1=...=\sigma_n=-1$ (if $\eps<f(x_0)$). We can then conclude 
    \begin{align*}
        R_n(\Fcal_{\mathrm{explain}}^{x_0}\mid A) \le f(x_0)+\eps <1 \quad (\text{resp. } \le -(f(x_0)-\eps) < 0)
    \end{align*}
    whereas 
    $$R_n(\Fcal_{\mathrm{predict}}^{x_0}\mid A)=1 \quad (\text{resp.} =0)$$
    as our function class $\Fcal_{\mathrm{predict}}^{x_0}$ can achieve the value 1 (resp. 0) for all Rademacher points. Using Lemma 21 of \citet{gunther2025informative}, it follows that the explanation is informative.
\end{proof}

\subsection{Proof of Theorem \ref{Thm-GAM-expl-var}}
\begin{proof}
    Define $g_j(x_j):=\E_{X\sim U}(\Phi_j(X)\mid X^{(j)}=x^{(j)})$ and consider the GAM $g(x):=g_1(x^{(1)})+...+g_d(x^{(d)}) + \E(f(X))$. 
	We then compute
		\begin{align*}
			\E_{X\sim U}((f(X)-g(X))^2) 
            &= \E_{X\sim U}\left(\left(\sum_{j=1}^d\Phi_j(X) - g_j(X^{(j)})\right)^2\right)\\
			&= \E_{X\sim U}\left(\sum_{i,j=1}^d \left(\Phi_i(X) - g_i\left(X^{(i)}\right)\right)\cdot\left(\Phi_j(X) - g_j\left(X^{(j)}\right)\right)\right)\\
			&= \sum_{i,j=1}^d \E_{X\sim U}\left( \left(\Phi_i(X) - g_i\left(X^{(i)}\right)\right)\cdot\left(\Phi_j(X) - g_j\left(X^{(j)}\right)\right)\right)\\
			&\le \sum_{i,j=1}^d \sqrt{\E_{X\sim U}\left( (\Phi_i(X) - g_i(X^{(i)}))^2\right)}\cdot \sqrt{\E_{X\sim U}\left( (\Phi_j(X) - g_j(X^{(j)}))^2\right)}\\
            &= \sum_{i,j=1}^d \sqrt{\E_{X\sim U}\left(\Var(\Phi_i(X)\mid X^{(i)})\right)}\cdot \sqrt{\E_{X\sim U}\left(\Var(\Phi_j(X)\mid X^{(j)})\right)}\\
			&\le d^2\eps.
		\end{align*}
\end{proof}

\subsection{A Version of the Reverse Statement of Theorem \ref{Thm-GAM-expl-var}}
\label{app:proof-GAM-expl-var}
\begin{theorem}
    If $f$ is close to a GAM $g$ w.r.t. $\Lcal^\infty$, which is $|f(x)-g(x)|\le \eps$ for all $x$ in the extended support, then the SHAP value of feature $j$ almost only depends on $x_j$ in the sense that
    $$\E\left(\Var\left(\Phi_j(X)\mid X^{(j)}\right)\right)\le 4\eps^2.$$
\end{theorem}

\begin{proof}
    By using our uniform bound, we can bound the difference of the value functions of $f$ and those of the GAM $g$ by 
		$$|v_f(S, x)-v_g(S,x)|=|\E_{X_{\bar{S}}}(f(x_S,X_{\bar{S}})-g(x_S,X_{\bar{S}}))| \le |\E_{X_{\bar{S}}}(\eps)| = \eps.$$
		Plugging this into the formula for the SHAP values of $f$ and $g$, we receive
		\begin{align*}
			\left|\Phi^f_j(x) - \Phi^g_j(x)\right| &= \left|\sum_{S\subseteq \{1,...,d\}\setminus j} c_S\cdot( (v_f(S\cup j)-v_f(S)) - (v_g(S\cup j)-v_g(S)))\right| \\
			&\le \sum_{S\subseteq \{1,...,d\}\setminus j} c_S\cdot\left(\left| v_f(S\cup j)-v_f(S)\right| + \left| v_g(S\cup j)-v_g(S)\right|\right)\\
			&\le \sum_{S\subseteq \{1,...,d\}\setminus j} 2c_S\eps\\
			&=2\eps
		\end{align*}
		as the coefficients $c_S=\frac{|S|!(d-|S|-1)!}{d!}$ sum up to 1 \citep{shapley1953value}. As the SHAP values of a GAM are its (mean-centered) component functions $g_j$ \citep{BorLux_shap_2023} and hence only depend on the regarding feature, this  particularly implies 
		\begin{align*}
			\E\left(\Var\left(\Phi^f_j(X)\mid X^{(j)}\right)\right) &= \E\left (\left(\Phi^f_j(X)-\E\left(\Phi^f_j(X)\mid X^{(j)}\right)\right)^2\right )\\
			&\le \E\left (\left(\Phi^f_j(X)- g_j\left(X^{(j)}\right)\right)^2\right )\\
			&= \E\left (\left(\Phi^f_j(X)- \Phi^g_j\left(X^{(j)}\right)\right)^2\right )\\
			&\le 4\eps^2.
		\end{align*}
\end{proof}

\section{Explanation Cards for SHAP Explanations}
\label{app:shap-expl-card}

Figure \ref{fig:explanation_card_shap_good_full} shows an example of a full SHAP explanation card with all features. \\
In Figure \ref{fig:explanation_card_shap_bad}, we see an example where the explanation card reveals that the dependence plot does not reflect the feature's effect on the prediction. 

\begin{figure*}
\tcbset{
  cardstyle/.style={
    enhanced,
    colback=white,
    colframe=black!60,
    boxrule=0.5pt,
    sharp corners,
  }
}
\begin{tcolorbox}[
    cardstyle,
    title={Explanation Card for SHAP Explanations in Medical Diagnosis, Targeting the Doctor (Part 1)},
    width=\textwidth
]

    \hlA{The patient is classified as sick with a predicted score of $0.565$. In general, patients with positive scores are classified as sick; the higher the score, the more certain the classification algorithm. The SHAP values for each feature are highlighted with a red dot in the plots below.} \\[0.5em]
    \hlB{The plots below show the SHAP values of patients in a local region around the explained instance: \\% and the range of the relevant feature is colored in the plot.
    }
    \begin{tabularx}{\linewidth}{rXrXrX}
        \hlB{Age:} &\hlB{29-62} & \hlB{Sex:} & \hlB{Male} & \hlB{Resting BP:} & \hlB{112-160 mmHG}\\
        \hlB{Cholesterol:} & \hlB{157-325 mg/dl} & \hlB{Fasting blood sugar:} & \hlB{$< 120$ mg/dl} & \hlB{Max heart rate:}  & \hlB{120-202}\\
        \hlB{ST depression:} & \hlB{0-3} & \hlB{ST slope:} & \hlB{upsloping or flat} 
    \end{tabularx}\\[0.5em]
    \hlC{The SHAP values or other patients in the local region (blue dots) and their regression (blue line) roughly reflect how age affects the prediction. The dotted lines show the maximal/minimal change in prediction if age is varied.} 
\tcblower

\begin{minipage}[t]{0.48\linewidth}
\begin{tcolorbox}[
    cardstyle,
    title={Feature: Age}
]
\includegraphics[width=\linewidth]{Figures/SHAP_card_CatBoostClassifier.png}
\end{tcolorbox}
\end{minipage}\hfill
\begin{minipage}[t]{0.48\linewidth}
\begin{tcolorbox}[
    cardstyle,
    title={Feature: Sex}
]
\includegraphics[width=\linewidth]{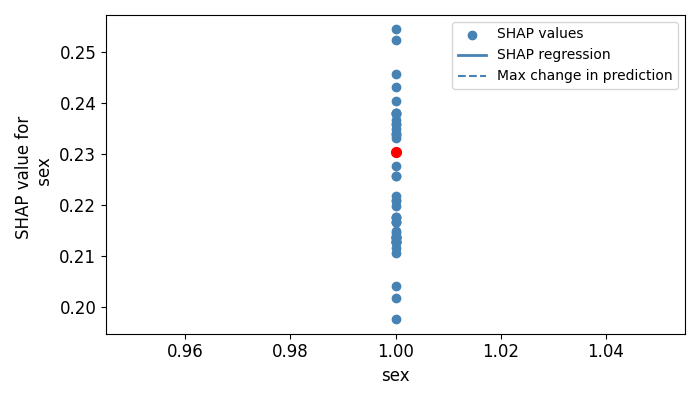}
\end{tcolorbox}
\end{minipage}

\begin{minipage}[t]{0.48\linewidth}
\begin{tcolorbox}[
    cardstyle,
    title={Feature: Type of Chest Pain}
]
\includegraphics[width=\linewidth]{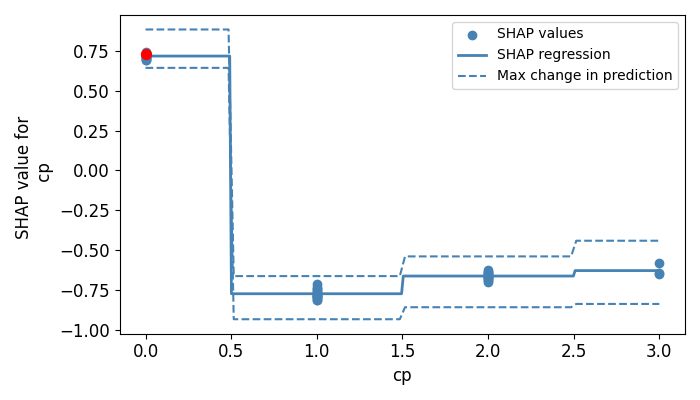}
\end{tcolorbox}
\end{minipage}\hfill
\begin{minipage}[t]{0.48\linewidth}
\begin{tcolorbox}[
    cardstyle,
    title={Feature: Resting Blood Pressure}
]
\includegraphics[width=\linewidth]{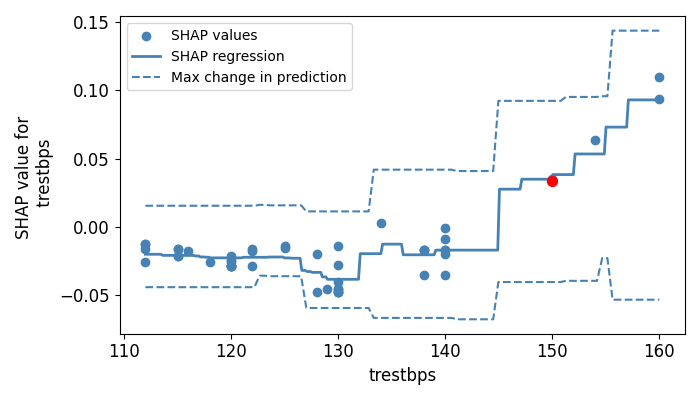}
\end{tcolorbox}
\end{minipage}

\begin{minipage}[t]{0.48\linewidth}
\begin{tcolorbox}[
    cardstyle,
    title={Feature: Cholesterol}
]
\includegraphics[width=\linewidth]{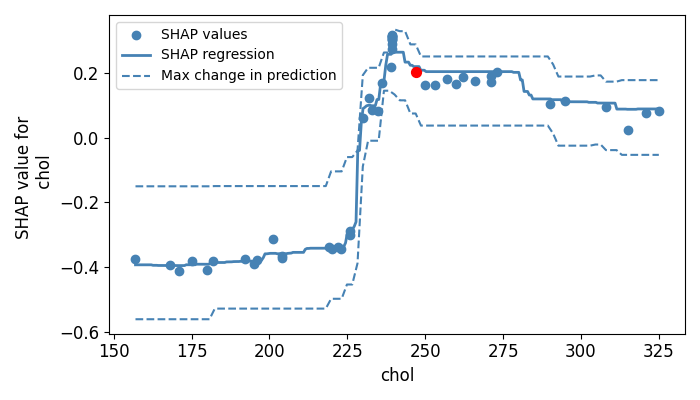}
\end{tcolorbox}
\end{minipage}\hfill
\begin{minipage}[t]{0.48\linewidth}
\begin{tcolorbox}[
    cardstyle,
    title={Feature: Fasting Blood Sugar}
]
\includegraphics[width=\linewidth]{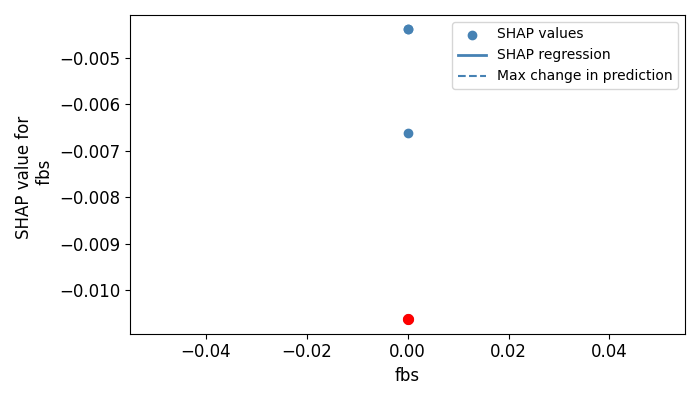}
\end{tcolorbox}
\end{minipage}

\end{tcolorbox}
\end{figure*}

\begin{figure*}
\tcbset{
  cardstyle/.style={
    enhanced,
    colback=white,
    colframe=black!60,
    boxrule=0.5pt,
    sharp corners,
  }
}

\begin{tcolorbox}[
    cardstyle,
    title={Explanation Card for SHAP Explanations in Medical Diagnosis, Targeting the Doctor (Part 2)},
    width=\textwidth
]

\begin{minipage}[t]{0.48\linewidth}
\begin{tcolorbox}[
    cardstyle,
    title={Feature: Resting Electrocardiographic Results}
]
\includegraphics[width=\linewidth]{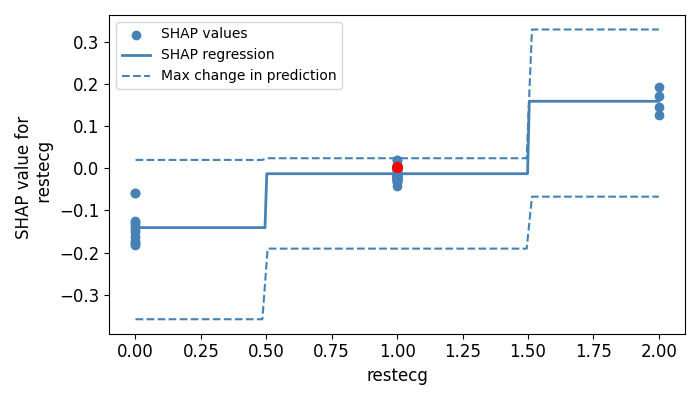}
\end{tcolorbox}
\end{minipage}\hfill
\begin{minipage}[t]{0.48\linewidth}
\begin{tcolorbox}[
    cardstyle,
    title={Feature: Maximum Heart Rate}
]
\includegraphics[width=\linewidth]{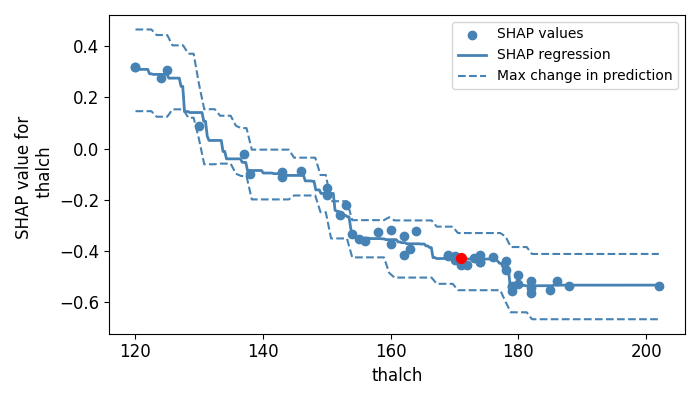}
\end{tcolorbox}
\end{minipage}

\begin{minipage}[t]{0.48\linewidth}
\begin{tcolorbox}[
    cardstyle,
    title={Feature: Exercise-induced Angina}
]
\includegraphics[width=\linewidth]{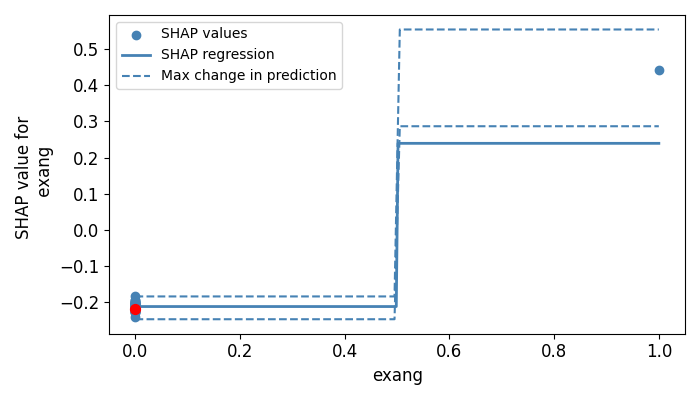}
\end{tcolorbox}
\end{minipage}\hfill
\begin{minipage}[t]{0.48\linewidth}
\begin{tcolorbox}[
    cardstyle,
    title={Feature: ST Depression}
]
\includegraphics[width=\linewidth]{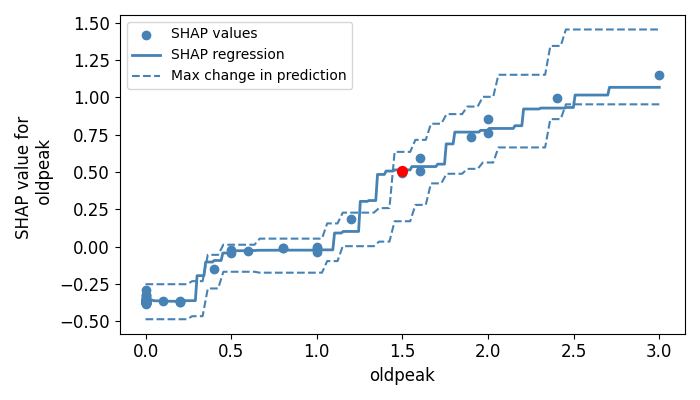}
\end{tcolorbox}
\end{minipage}

\begin{minipage}[t]{0.48\linewidth}
\begin{tcolorbox}[
    cardstyle,
    title={Feature: ST Slope}
]
\includegraphics[width=\linewidth]{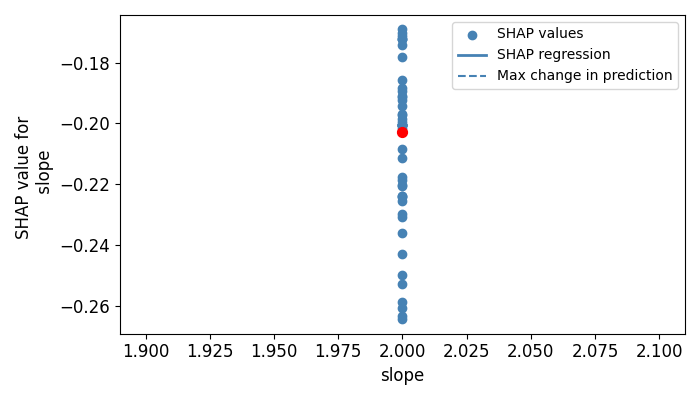}
\end{tcolorbox}
\end{minipage}\hfill
\begin{minipage}[t]{0.48\linewidth}
\begin{tcolorbox}[
    cardstyle,
    title={Feature: Number of Major Vessels}
]
\includegraphics[width=\linewidth]{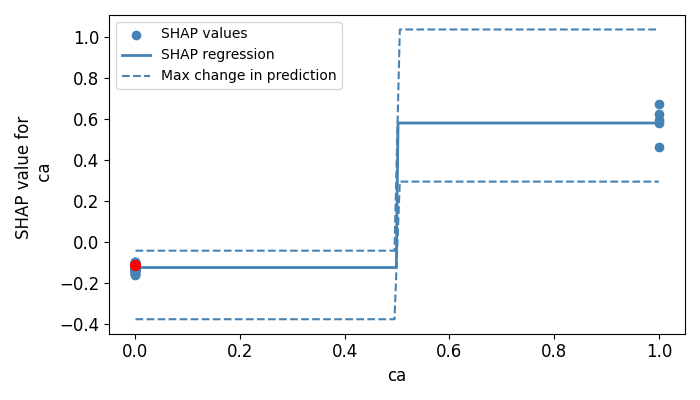}
\end{tcolorbox}
\end{minipage}

\begin{minipage}[t]{0.48\linewidth}
\begin{tcolorbox}[
    cardstyle,
    title={Feature: Thalassemia}
]
\includegraphics[width=\linewidth]{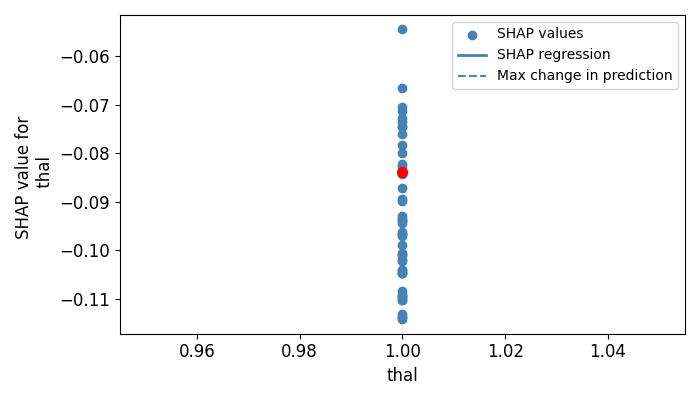}
\end{tcolorbox}
\end{minipage}

\end{tcolorbox}

\caption{\legendbox{color1} Machine decision and SHAP value. \;
\legendbox{color2} Local region. \;
\legendbox{color3} Feature effect on prediction. \\ \\
{\bf Explanation Card for a SHAP Explanation for a doctor debugging a prediction:} 
Here, we see how the explanation card could look like if all features are included.}
\label{fig:explanation_card_shap_good_full}
\end{figure*}

\begin{figure*}[t]
\tcbset{
  cardstyle/.style={
    enhanced,
    colback=white,
    colframe=black!60,
    boxrule=0.5pt,
    sharp corners,
  }
}
\begin{tcolorbox}[
    cardstyle,
    title={Explanation Card for SHAP Explanations in Medical Diagnosis, Targeting the Doctor},
    width=\textwidth
]

    \hlA{The current patient is classified as healthy, their probability of being sick is 0.31. This probability is considered in the following.} \\[0.5em]
    \hlB{The analysis was done in the following local region around the patient:\\% and the range of the regarding feature is colored in the plot.
    }
    \begin{tabularx}{\linewidth}{rXrXrX}
        \hlB{Age:} &\hlB{29-62} & \hlB{Sex:} & \hlB{Male} & \hlB{Resting BP:} & \hlB{112-160 mmHG}\\
        \hlB{Cholesterol:} & \hlB{157-325 mg/dl} & \hlB{Fasting blood sugar:} & \hlB{$< 120$ mg/dl} & \hlB{Max heart rate:}  & \hlB{120-202}\\
        \hlB{ST depression:} & \hlB{0-3} & \hlB{ST slope:} & \hlB{upsloping or flat} 
    \end{tabularx}
\tcblower
\begin{tcolorbox}[
    cardstyle,
    title={Feature: Age},
    sidebyside,
    sidebyside gap=3mm,
    righthand width=0.55\textwidth
]
\includegraphics[width=\linewidth]{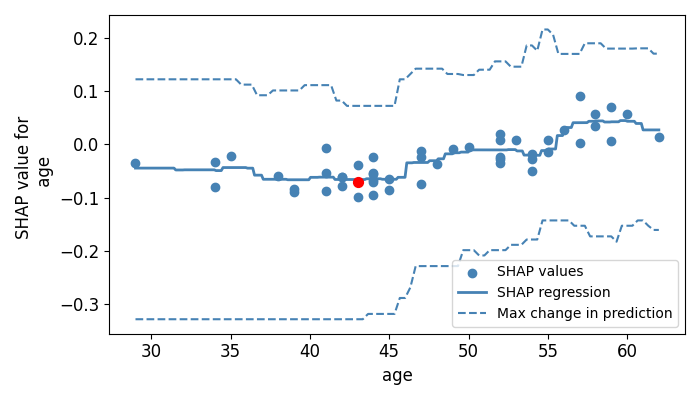}
\tcblower
\hlA{SHAP attributes a value of $-0.07$ to the patient's age ($43$), see red dot in the plot.}\\[0.5em]
\hlC{The SHAP values or other patients in the local region (blue dots) and their regression (blue line) roughly reflect how age affects the prediction. The dotted lines show the maximal/minimal change in prediction if age is varied.} %
\end{tcolorbox}
\end{tcolorbox}
\caption{\legendbox{color1} Machine decision and SHAP value. \;
\legendbox{color2} Interpretable region. \;
\legendbox{color3} Feature effect on prediction. \\ \\
{\bf Explanation Card for a SHAP Explanation for a doctor debugging a prediction:} Here, the explanation card simply reveals that the SHAP dependence plot cannot be related to the prediction behavior: The envelope defined by the dotted lines is very large.}
\label{fig:explanation_card_shap_bad}
\end{figure*}

\section{Data, Models, and Implementation Details}

\subsection{Counterfactual Explanations}

This section documents the data sources, models, and computational procedures used in the experiments for the section on counterfactual explanation, with the goal of clarifying methodological choices.

\paragraph{Datasets.} Two primary datasets were used. First, we used the ACSIncome dataset from the Folktables benchmark \citep{ding2021retiring}, derived from the American Community Survey (ACS). Data were obtained using the folktables interface via ACSDataSource, with the person survey, 1‑Year horizon, for California (CA) in 2018. The prediction task is binary income classification, where the label indicates whether an individual’s income exceeds the threshold defined by ACSIncome. Features include demographic and socioeconomic attributes such as age, hours worked per week, education level, marital status, and related categorical and ordinal variables.

Second, we used the South German Credit dataset \citep{south_german_credit_522, gromping2019south}. This dataset contains financial and demographic attributes (e.g., credit amount, duration, age, employment duration, housing status), with a binary credit risk label. 

The datasets were primarily used to illustrate counterfactual reasoning and stability region computations around data points and their counterfactual explanations. The data were split into training and test sets using a standard random train–test split.

\paragraph{Models.}

We conducted experiments using two types of models.

\begin{itemize}
    \item Decision tree classifier: Used for visualization, as well as for counterfactual and stability‑region analysis. Decision trees were particularly important because their axis‑aligned splits allow explicit extraction of decision path constraints and leaf‑level feature bounds.
    \item Random forest classifier: Used mainly for visualization purposes.
\end{itemize}

\paragraph{Counterfactual Computation.} Counterfactual explanations were generated using the DiCE (Diverse Counterfactual Explanations) library \citep{mothilal2020dice}. They were obtained using the genetic algorithm–based method in DiCE, typically requesting a single counterfactual that flips the predicted class from negative to positive while respecting feature constraints. The resulting counterfactual instance serves as the target point for subsequent stability analysis.

\paragraph{Stability Analysis.} Regions of stability were defined relative to a decision tree model and a reference point. For a given input, the tree’s decision path was extracted. Each internal node on the path induces a constraint of the form feature $\leq$ threshold or feature $>$ threshold. Aggregating these constraints yields an axis‑aligned box in the feature space.

Two types of safe regions were considered:
\begin{itemize}
    \item Region of stability: For a counterfactual explanation, the region of stability was obtained by taking the positive leaf of the decision tree to which the counterfactual belongs. Any perturbation within this region preserves the positive prediction.
    \item Region of stability around the input: For a negatively classified original instance, a stable region is obtained by taking the region corresponding to the negative leaf of the data point and intersecting it with the stability region of the counterfactual shifted by the vector $(x_0-x_C)$. Data points within this region receive negative predictions, and the same counterfactual action $(x_C-x_0)$ works for them.
\end{itemize}
When analyzing transitions between an original negative point and a positive counterfactual, safe regions were further restricted to only investigating features that a person might possibly change.

\paragraph{Visualization and Plots.} To visualize decision boundaries and safe regions, two features of interest (age and educational attainment) were selected. All other features were fixed at their median values. Original points, counterfactuals, and safe regions were overlaid to illustrate stability and local stability geometrically. In Figure \ref{fig:random_forest_vs_decision_tree}, the middle panel depicts the result of a decision tree of depth 10 trained on ACSIncome. The right panel is the result of a standard random forest classifier with 20 trees trained on the same data. 

\paragraph{Additional Implementation Details.} All experiments were implemented in Python using standard scientific libraries (NumPy, pandas, scikit‑learn, Matplotlib), along with specialized packages such as Folktables and DiCE.

\subsection{SHAP Explanations}
\label{app:implementation-shap}

\paragraph{Datasets.} The dataset used throughout the SHAP section is the UCI heart disease dataset \citep{heart_disease_45}, downloaded from kaggle \citep{heart_disease_uci_kaggle}, where the data from all hospitals is merged into one dataset. The label of the dataset is either $0$ (no heart disease) or one of the numbers $1,2,3,4$ (state of present heart disease). We converted any of the positive integers into a $1$ to make the problem a binary classification task with labels 0 (healthy) and 1 (sick). Missing values were replaced by the median (if numeric) or the most frequent feature (if categorical). Furthermore, we dropped the features ``dataset'' and ``id'' as they only indicate the hospital, respectively, the ID of the patient. The data was split into a training and a test set using a standard random train-test split.

\paragraph{Models.}
For creating Figures \ref{fig:shap-dp}, \ref{fig:shap-illustration}, \ref{fig:explanation_card_shap_good}, \ref{fig:SHAP-DP bad with envelope}, \ref{fig:explanation_card_shap_good_full} and \ref{fig:explanation_card_shap_bad}, we used 3 different models. A random forest with 100 trees of unlimited depth gave rise to the unstable plot in Figure \ref{fig:SHAP-DP chaotic} as well as the negative card in Figure \ref{fig:explanation_card_shap_bad}. However, each plot utilized a different point with its neighbors. We furthermore trained a boost with 100 trees of depth 3, leading to Figure \ref{fig:shap-dp}, where we plotted a dependence plot of all data points. The same model was used for the plots \ref{fig:SHAP-DP bad ICE curve} and \ref{fig:SHAP-DP-angina}, as well as the plots in our explanation cards in Figure \ref{fig:explanation_card_shap_good} and \ref{fig:explanation_card_shap_good_full}. There, we only considered a data point with its 50 nearest neighbors. The additive model represented in Figure \ref{fig:SHAP-DP GAM} was obtained by training a catboost with 300 trees of depth 1.\\

\paragraph{Computation of SHAP.}
For all SHAP values that were computed in this paper, we used the official package SHAP \citep{Lundberg17}. More concretely, we used TreeExplainer in its interventional version (which is the default when a background dataset is given) on the raw model output. 

\paragraph{Explanation Card.} 
For creating the explanation card, we first pick the 50 nearest neighbors of
the point of interest and compute their SHAP values. The neighbors are
determined with respect to the maximum norm; as such, they define a
rectangle, which we use as our local region $U$. 
Within the region $U$, we plot, for every feature, the SHAP values of the neighbors in a plot depending on their regarding feature, where the individual instance is highlighted. We regress each SHAP value $\Phi_j$ with $X^{(j)}$, using a random forest with 100 trees and a maximal depth of 3. This leads to components of an approximate additive model, which we plotted as the blue line in the plot of the explanation card. To get the dotted lines that represent a lower and upper bound, we proceed as follows: For each of the points in $U$, we vary the feature within the range that $U$ dictates and consider the resulting prediction. This leads to a curve $x^{(j)}\mapsto f(x^{(j)}, x_i^{(\bar{j})})$ for a point $x_i$. We shift the curve such that it goes through the regarding point in our SHAP plot, that is, we consider the curve $c_i:x^{(j)}\mapsto f(x^{(j)}, x_i^{(\bar{j})}) - f(x_i)+\Phi_j(x_i)$. As our dotted lines, we then plot the maximum, respectively, the minimum of these curves over all datapoints in $U$. Mathematically, these are the curves $x^{(j)}\mapsto \max_{x_i\in U} c_i(x^{(j)})$ and $x^{(j)}\mapsto \min_{x_i\in U} c_i(x^{(j)})$. That means that if we take any of the datapoints in $U$, vary its feature $j$, then the change in prediction lies in between the two dotted lines.  \\
With the plot, we deliver short phrases explaining what can be derived from the
plot, including the prediction and SHAP value of the instance, the feature ranges of the rectangle $U$ as well as a an explanation of the plot. Due to space constraints, we only present this for the feature ``age'' as an example.

\paragraph{Additional Implementation Details.} All experiments were implemented in Python using standard scientific libraries (NumPy, pandas, scikit‑learn, Matplotlib), along with specialized packages such as Catboost and SHAP.

\end{document}

%% file: latex_definitions/local_latex_stuff.tex
\usepackage{caption,subcaption}
\usepackage[most]{tcolorbox}
\usepackage[dvipsnames]{xcolor}
\usepackage{soul}
\usepackage{hyperref}
\hypersetup{
    colorlinks=true,
    allcolors=black
}
\usepackage{tabularx}

\usepackage{placeins} % makes sure the figure stay within the section, used in appendix with command \FloatBarrier

\usepackage{comment}
\usepackage{amsthm,amsmath,amsfonts,amssymb,mathtools,stmaryrd,mleftright}

\newtheorem{theorem}{Theorem} 
\newtheorem{definition}[theorem]{Definition}

\definecolor{color1}{RGB}{238, 221, 221}  % dusty rose
\definecolor{color2}{RGB}{221, 235, 221}  % sage
\definecolor{color3}{RGB}{221, 231, 241}  % steel blue
\definecolor{color4}{RGB}{245, 241, 221}  % warm cream
\definecolor{color5}{RGB}{235, 225, 238}  % muted mauve

\sethlcolor{color1}
\newcommand{\hlA}[1]{\sethlcolor{color1}\hl{#1}}
\newcommand{\hlB}[1]{\sethlcolor{color2}\hl{#1}}
\newcommand{\hlC}[1]{\sethlcolor{color3}\hl{#1}}
\newcommand{\hlD}[1]{\sethlcolor{color4}\hl{#1}}

\newcommand{\legendbox}[1]{\colorbox{#1}{\strut\hspace{0.8em}}}

\DeclareMathOperator{\E}{E}
\DeclareMathOperator{\Var}{Var}

\newcommand{\eps}{\varepsilon}

\newcommand{\R}{\mathbb{R}}

\newcommand{\Lcal}{\mathcal{L}}
\newcommand{\Fcal}{\mathcal{F}}
\newcommand{\Ecal}{\mathcal{E}}
\newcommand{\Xcal}{\mathcal{X}}